\documentclass{article}

\usepackage{times}
\usepackage{graphicx}
\usepackage{latexsym}
\usepackage{amsmath}
\usepackage{amssymb}
\usepackage{xspace}
\usepackage{complexity}

\usepackage{algorithmic}
\usepackage{algorithm}

\usepackage{rotating}

\usepackage{tikz}

\usetikzlibrary{arrows}
\usetikzlibrary{shapes}

\newtheorem{theorem}{Theorem}
\newtheorem{lemma}{Lemma}
\newtheorem{example}{Example}

\newtheorem{proof}{Proof}

\newcommand{\todo}[1]{{\large\bf TODO: }#1\ensuremath{\Box}}
\newcommand{\nop}[1]{}

\def\QED{{\phantom{x}} \hfill $\Box$}

\begin{document}

\title{Auction optimization with models learned from data}

\author{Sicco Verwer\thanks{Delft University of Technology, The Netherlands; Email: \texttt{ S.E.Verwer@tudelft.nl}.}  \and Yingqian Zhang\thanks{Department of Econometrics; Erasmus University Rotterdam; The Netherlands; Email: \texttt{yqzhang@ese.eur.nl}.} \and Qing Chuan Ye\thanks{Department of Econometrics; Erasmus University Rotterdam; The Netherlands; Email: \texttt{ye@ese.eur.nl}}}
\date{}

\maketitle
\begin{abstract}

In a sequential auction with multiple bidding agents, it is highly challenging to determine the ordering of the items to sell in order to maximize the revenue due to the fact that the autonomy and private information of the agents heavily influence the outcome of the auction.  

The main contribution of this paper is two-fold. First, we demonstrate how to apply machine learning techniques to solve the optimal ordering problem in sequential auctions.
We learn regression models from historical auctions, which are subsequently used to predict the expected value of orderings for new auctions. Given the learned models,
we propose two types of optimization methods: a black-box best-first search approach, and a novel white-box approach that maps learned models to integer linear programs (ILP) which can then be solved by any ILP-solver. Although the studied auction design problem is hard, our proposed optimization methods obtain good orderings with high revenues.

Our second main contribution is the insight that the internal structure of regression models
can be efficiently evaluated inside an ILP solver for optimization purposes. To this end, we provide efficient encodings of regression trees and linear regression models as ILP constraints. This new way of using learned models for optimization is promising. As the experimental results show, it significantly outperforms the black-box best-first search in nearly all settings.

\end{abstract}

\section{Introduction}

One of the main challenges of mathematical optimization is to construct a mathematical model describing the properties of a system. 
When the structure of the system cannot be fully determined from the knowledge at hand, machine learning and data mining techniques have been used in optimization instead of such models.
They have, for example, been used as decision values~\cite{Gabel08}, fitness functions~\cite{huyet06}, or model parameters~\cite{LiO05}. To the best of our knowledge, so far models learned by data mining or machine learning have only been used for optimization in a black-box manner, e.g., using only the predictions of learned models but not their internal structure. In contrast to such black-box approaches, in this paper, we develop a novel white-box optimization method, that is, we map entire models to a set of integer linear programming constraints such that the internal structure of the models is used during problem solving.
The proposed white-box method together with a black-box method also provide a solution to an optimization problem of key interest to the artificial intelligence and operations research communities: auction design. We briefly introduce this domain before going into the details of our methods.

\subsection{Sequential auction design}

Auctions are becoming increasingly popular for allocating resources or items in business-to-business and business-to-customer markets. 
Often sequential auctions~\cite{Bernhardt} are adopted in practice, where items are sold consecutively to bidders. Sequential auctions are in particular desirable when the number of items for sale is large (e.g.,~flower auctions~\cite{Heck97experiences}), or when the buyers enter and leave the auction dynamically (e.g.,~online auctions~\cite{Pinker10}).
In a sequential auction, an auctioneer may tune several auction parameters to influence the outcome of an auction, such as reserve prices for items and in which order to sell them. In other words, (s)he can design auctions for the purpose of achieving some predefined goal.
In this paper, we solve one specific auction design problem, namely, deciding the optimal ordering of items to sell in a sequential auction in order to maximize the expected revenue (OOSA in short). We assume bidders in such auctions are budget constrained. This is a highly relevant problem in today's auctions since bidders almost always have limited budget, as seen for instance in industrial procurement~\cite{Gallien05smart}. 
Previous research has shown that with the presence of budget constraints, the revenue collected by the auctioneer is heavily dependent on the  ordering of items to sell~\cite{Elmaghraby03importance,Grether09,Raviv06}. This holds already for a toy problem with 2 items.
Let us use a simple example to illustrate the importance of ordering in such cases. 
\begin{example}\label{exp1}
Two agents $A_1$ and $A_2$ take part in a sequential auction of items. For sale are items $r_1$ and $r_2$. The values of these items for the agents are given as follows: $v_{A_1}(r_1)=5, v_{A_1}(r_2) = 6, v_{A_2}(r_1) = 0, v_{A_2}(r_2) = 5$.
In addition, both $A_1$ and $A_2$ have a budget limit of $5$. Assume the reserve prices (i.e., the lowest price at which the auctioneer is willing to sell) for both items are 4.

Suppose the items are sold by means of first-price, English auction\footnote{The English auction that we consider is the one where the starting price is the reserve price, and bidders bid openly against each other. Each subsequent bid should be higher than the previous bid, and the item is sold to the highest bidder at a price equal to her bid.}. We assume a simple bidding strategy in this example. Each agent $A_j$'s reservation price (i.e., the highest price $j$ is willing to pay) for item $r_i$ is $v_{A_j}(r_i)-1$. Agents bid myopically on each item, i.e., they will bid up to their reservation price, if their budgets allow.\footnote{We will discuss different bidding strategies in Section~\ref{sec:exp}.} The auctioneer's goal is to maximize the revenue, i.e., the total sold price of the items.

Consider one situation where the auctioneer sells  first $r_2$ and then $r_1$.
$A_1$ will get $r_2$ when she bids 5, and $r_1$ will not be sold since $A_2$ is not interested in it and $A_1$ is out of budget. The total revenue is $5$. However, if the selling sequence is ($r_1$, $r_2$), $A_1$ will win $r_1$ with the reserve price 4, and then $A_2$ will win $r_2$ with price 4. 
The collected revenue is 8 in this case.
\QED
\end{example}

Most of the current approaches to the ordering problem in sequential auctions assume a very restricted market environment. They either study the problem of ordering two items, see~\cite{Subramaniam09optimal,Pitchik09budget}, or a market with homogeneous bidders~\cite{Elkind07maximizing}. To the best of our knowledge, we are the first to consider how to order items for realistic auction settings with many heterogeneous bidders competing for many different items. 
This problem is highly complex---a good design on ordering needs to take care of many uncertainties in the system.
For instance, in order to evaluate the revenue given an ordering, the optimization algorithm needs to know the bidders' budgets and preferences on items, which is usually private and unshared. Furthermore, the large variety of possible bidding strategies that bidders may use in auctions are unknown. 
This auction design problem is a typical example where the mathematical optimization model cannot be fully determined, and hence, machine learning and data mining techniques can come into play. This is exactly what our approach builds upon.

\subsection{Learning models for white-box and black-box optimization} Nowadays more and more auctions utilize information technology, which makes it possible to automatically store detailed information about previous auctions along with their orderings and the selling price per auctioned item.
Our approach to solving the problem of optimal ordering for sequential auctions starts with the historical auction data. We define and compute several relevant features and then use them to learn 
regression trees and linear regression models
for the expected revenue. 
Given the models, we propose two approaches to find the optimal ordering for a new set of items:
(1) a best-first search that uses the models as a black-box to evaluate different orderings of the items;
and (2) a novel white-box optimization method that 
translates the models and the set of items into a mixed-integer program (MIP) and runs this in an ILP-solver (CPLEX). Figure~\ref{fig:white-black} displays the general framework of our approaches using these two optimization methods.

\begin{figure}[thb]
\begin{center}
\includegraphics[width=.8\columnwidth]{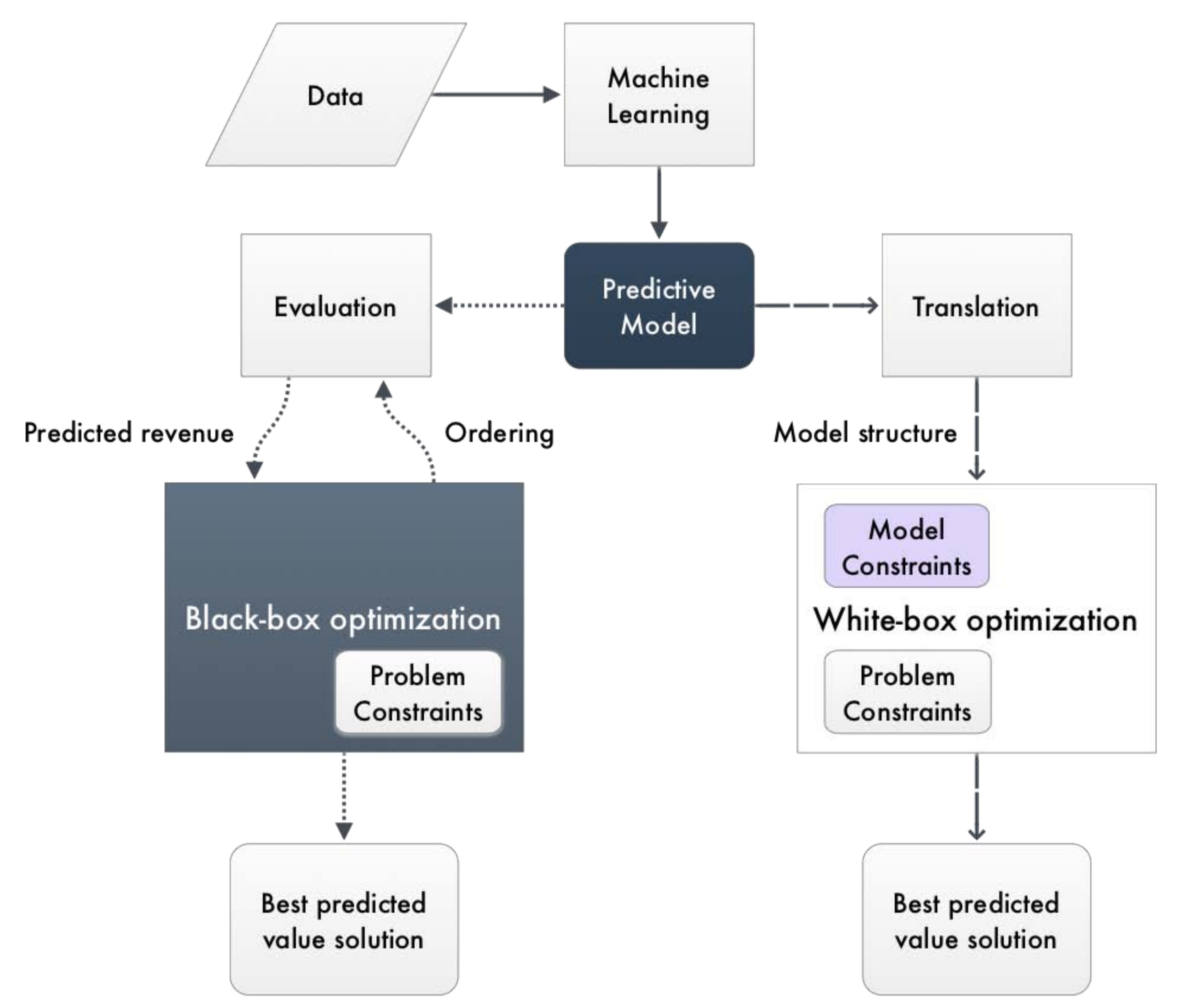}
\end{center}
\caption{Solving OOSA using white-box optimization and black-box optimization with learned models. Black-box optimization only calls the predictive model to evaluate possible orderings. White-box optimization translates the internal structure of the predictive model to MIP constraints.\label{fig:white-black}}
\end{figure}

Just like the traditional black-box optimization approach (see, e.g.~\cite{jones1998efficient, shan2010survey}), our best-first search is ignorant of the internal structure of the models and only calls it to perform function evaluations, i.e., predicting the revenue of an ordering of the items. 
Optimization is possible by means of a search procedure that 
uses heuristics to produce new orderings depending on previously evaluated ones. Our best-first search makes use of dynamic programming cuts inspired by sequential decision making in order to reduce the search space. 

One of the main contributions of this paper is the realization that learned models can be evaluated efficiently inside modern mathematical optimization solvers.
This evaluation includes the computation of feature values (the input to machine learning), the evaluation of these features using a learned model (the output from machine learning), and a possible feedback from such evaluations to new features. 
In this paper, we efficiently translate all of these steps for two types of learned models (regression trees and linear regression models) into mixed-integer constraints. The resulting mixed-integer program can then be evaluated in any modern integer linear programming (ILP) solver. We call the resulting method white-box optimization with learned models.

In this way, modern exact solvers can be used instead of a heuristic search. These solvers use (amongst others) advanced branch-and-bound methods to cut the search space, compute and optimize a dual solution, and can prove optimality without testing every possible solution. This is the main benefit of using the white-box method over a black-box one.
The downside, however, is that when the learned model is complex, the white-box method may lead to a large mathematical model that is difficult to optimize. 
We compare these two approaches and investigate this trade-off by applying them to the OOSA problem.

\paragraph{Contributions and organization}

Although we use sequential auction design to illustrate our method, all of our constructions are general and can be applied to any optimization setting where unknown relations can be represented using regression models that have been learned from data. The only constraint for using the white-box method is that the feature values need to be computable using integer linear functions from intermediate solutions. Our approach can thus be applied to complex optimization settings where entire orders, schedules, or plans need to be constructed beforehand.

We list our contributions as follows:

\begin{itemize}
\item We demonstrate how to apply regression methods from machine learning to OOSA.
\item We give an efficient encoding of regression trees and linear regressors into MIP constraints.
\item We prove OOSA with budget constrained bidders to be NP-hard, also when using these regression models.
\item We provide the first method that tackles OOSA in realistic settings.
\item We automate the construction of mathematical models using machine learning.
\end{itemize}

In Section~\ref{sec:auction}, we formally introduce the problem of optimal ordering for sequential auctions (OOSA), and then we show how to learn regression models from historical auction data in Section~\ref{sec:representing} using standard machine learning methods. Based on the learned models, our white-box optimization method and a black-box optimization are introduced to find the optimal ordering for OOSA in Section~\ref{sec:translate}. Extensive experiments are presented in Section~\ref{sec:exp} where we compare the performance of the two proposed optimization methods using both the learned models and the auction simulator. Before we conclude, we compare and discuss more related works in Section~\ref{sec:related}.

\section{Optimal ordering for sequential auctions OOSA}\label{sec:auction}


We assume there is a finite set of bidders (or agents). Let $R = \{r_1,\ldots,r_l\}$ denote the collection of
the item types, and the quantity of each item type can be more than 1.  
When it is clear from the context, we will slightly abuse the notation and use $I=\{r_1,r_2\ldots,r_1,\ldots\}$ to denote the multiset of all available items.
Each bidder agent $i$ has a valuation (or preference) for each type of item
$v_i:R \rightarrow \mathbb{R^+}$.
In addition, each agent has a budget $b_i$ on purchasing items, and (s)he desires to win as many items as are being auctioned within the budget limit.

In one auction, a set of items $I$ with type set $R'\subseteq R$ will be auctioned sequentially using a predetermined order. 
For example, given types $r_1$ and $r_2$ with quantities of 1 and 2 respectively, there are three possible orderings of resources: $(r_1,r_2,r_2)$, $(r_2,r_1,r_2)$, and $(r_2,r_2,r_1)$.
For each $r_j$ that is being auctioned, agent $i$ puts a bid on $r_j$ based on its valuation function if its remaining budget allows. 
The agent who bids highest on $r_j$ wins $r_j$. This sequential auction ends when all items have been auctioned, or when all agents run out of their respective budgets.

We assume that such an auction is repeated over time, and 
each auction sells possibly different items $I$.
At the end of each auction, we have the following information at our disposal: (1) the ordering of auctioned items; and (2) the allocation of items to agents with their winning bids. The optimization problem we study is: \emph{given a set of items and budget constrained bidders, finding an optimal ordering of items in sequential auctions such that the expected revenue is maximized}. We call it OOSA.
We now show that the decision version of this optimization problem is NP-hard, even if we have complete information on bidders' preferences and they are not strategic (i.e., they bid truthfully according to their true preferences). 

\begin{theorem}\label{thrm:complexity} Given a set of items $R$,  preferences $v_i:R \rightarrow \mathbb{R^+}$, and budgets $b_i$ for every bidder $i$. The problem of deciding whether there exists an ordering that obtains a revenue of at least $K \in \mathbb{R^+}$ is NP-hard.
\end{theorem}

\begin{proof}

By reduction from the well-known NP-hard partition problem~\cite{Garey79}: Given a set of integers $I = \{ i_1, \ldots, i_n \}$, is $I$ dividable into two sets $A$ and $B$ such that $\sum A = \sum B$? We need two bidders with preferences such that $v_1(r_k) = 2 \cdot i_k $ and $v_2(r_k) = 2 \cdot i_k + 1$ for $1 \leq k \leq n$. The reserve price for each item $k$ ($1 \leq k \leq n$) is $i_k$. The agents' budgets are $b_1 = \frac{1}{2} \sum I$ and $b_2 = \sum I$. The set of items $R$ is $\{r_1,\ldots,r_n\}$ and $K$ is $\frac{3}{2} \sum I$. We claim that $I$ is partitionable into two sets with equal sums if and only if there exists an ordering that obtains a revenue of $K$ (or more).

$(\Rightarrow)$ Given a partition of $I$ into sets $A$ and $B$, we sell all items in $A$ first, and those in $B$ later. In this case, agent $2$ will buy all items in $A$ with price $2\cdot i_k$ as it is the minimal bid to win the items from agent $1$. After buying all items in $A$, agent $2$ will have spent $2 \cdot \sum_{i_k \in A} i_k$, which makes $\sum I$ in total (since $\sum_{i_k \in A} i_k = \frac{1}{2} \sum I$). This is the entire budget of agent $2$. All items in $B$ are then sold to agent $1$ with the reserve price $i_k$. Thus agent $1$ pays $\sum_{i_k \in B} i_k = \frac{1}{2} \sum I$. This makes a total revenue of $ \sum I + \frac{1}{2}\sum I = \frac{3}{2} \sum I = K$.

$(\Leftarrow)$ Suppose we have an ordering such that agent $1$ and $2$ spend all of their budget ($K$ in total). This means that agent $2$ wins the first set of items A, each costing $2 \cdot i_k$ till it uses all its budget. Thus we have $2 \cdot \sum _{k \in A} i_k = \sum I $, i.e., $ \sum _{k \in A} i_k = \frac{1}{2}\sum I $. Agent $1$ pays $i_k$ for the remaining items in $B$, $ B= I\setminus A$, and it uses all its money: $\sum _{k \in B} i_k = \frac{1}{2} \sum I$. Hence, we have a partition of $I$ where $\sum A = \sum B$.

The construction is clearly polynomial time.
\end{proof}


Several related works deal with this type of ordering optimization problem. 
For example, in the Economics literature, the authors of \cite{Subramaniam09optimal} investigate the optimal ordering strategy for the case where the auctioneer has two items to sell. They show that when the items are different in value, the higher valued items should be auctioned first in order to increase the seller's revenue. Pitchik~\cite{Pitchik09budget} points out that in the presence of budget constraints, in a sealed-bid sequential auction, if the bidder who wins the first good has a higher income than the other one, the expected revenue is maximized.
With these greatly simplified auction settings, it is possible to derive bidders' equilibrium bidding strategies. Hence, with some assumptions on the distributions of bidders' budgets and valuations of items, the optimal ordering can be theoretically derived. However, as real-world auctions are much more complex and uncertain in terms of the sizes of items/bidders, agents' preferences and bidding strategies, these existing results cannot be applied.
In this paper, we instead focus on learning the overall behaviors of the group of bidders from historical auction data by machine learning  techniques, as the first step of solving   OOSA.

In order to apply ML techniques, we assume in every sequential auction the set of participating bidders and their characteristics (preferences, budgets, bidding strategies) to be similar.
This simplifies the problem of learning a good ordering. Instead of learning the individual valuations/budget/bidding strategies of agents, we can treat the agent population as a single entity for which we need to find a single global function.
 Obviously, such an approach will fail if the agents are radically different in every auction. However we consider this assumption sensible in many auction settings such as industrial procurement auctions where the same companies repeatedly join the auctions with similar interests, and the Dutch flower auction where there can be different bidders every day, but it seldom occurs that one day bidders are only interested in roses and the next day they only want tulips.
Although the different participants can be interested in different item types, the interests of the group of participants remain stable.

\section{Learning predictive models for OOSA}\label{sec:representing}


At the end of each sequential auction, we have the following information at our disposal: (1) the ordering of auctioned items; and (2) the revenue of each sold item.
We need to find a suitable way to model the expected revenue of such orderings. An ordering can be thought of as a sequence of items. However, to the best of our knowledge none of the existing sequence models fit our auction setting, see also Section~\ref{sec:mdp}.  What comes closest to our auction setting are models such as Markov decision processes (MDPs)~\cite{mdp}. These directly model the expected price per item and come with methods that can be used to optimize the expected total reward (revenue). However, an implementation of the auction design problem as an MDP is not easy. We discuss this in more detail in Section~\ref{sec:mdp}.

\subsection{A regression model for OOSA}
Instead of representing our problem with a sequence model such as an MDP, we view the prediction of an auction's outcome as a regression problem. We split this problem into the subproblems of predicting the value of the auctioned items. We then sum these up to obtain the overall objective function:
\[
V(r_1 \ldots r_n) = \sum_{1 \leq k \leq n} R(r_k,\{r_j \mid j < k\},\{r_l \mid k < l\}),
\]
where $R(r_k,J,L)$ is a regression function that determines the expected value of $r_k$ given that $J$ was auctioned before and $L$ will be auctioned afterwards. 
The main benefit of this representation is that modern machine learning methods can be used to learn this function $R$ from data. In addition, since every item sold represents a single sample, every auction contains many samples that can be used for learning, further reducing the amount of required data.

\subsection{The regression functions}
In this paper, we use regression trees~\cite{cart84} and least absolute shrinkage and selection operator (LASSO)~\cite{Tibshirani94regressionshrinkage} as regression functions, 
and train them using features based on the items auctioned before and after the current item. We first briefly introduce these regressors.

Regression trees are a form of decision trees where the predicted values are real numbers. A decision tree is one of the most popular predictive models for mapping feature values to a target value. It is a tree-shaped graph with a root node, interior nodes, and leaf nodes. The root and every interior node contains a Boolean test for a specific feature value $f$, such as $f > 5$. Every leaf node contains an output value $v$. It maps the feature values to an output by performing all the tests along a path from the root to a leaf. For every test performed, if the outcome is true ($f > 5$), the path is continued along the left branch, if the outcome is false ($f \leq 5$), the path is continued along the right branch. Once a leaf is reached, it outputs the value it contains $v$.

A regression tree learner aims to find a tree that minimizes the mean squared error of the predicted and the actual observed values. Most regression tree learning algorithms follow a greedy strategy that splits interior nodes as long as the decrease in error is significant. A split replaces one leaf node by an interior node connected with two new leaf nodes. The interior node receives as Boolean constraint one that minimizes the mean-squared error of the resulting tree, where the leaf nodes predict the mean value of all observed data values that end up in that leaf after mapping all data samples to leaf nodes.

LASSO is a method for constructing a linear regression function $v(f_1,\ldots,f_m) = c_1f_1 + c_2f_2 + \ldots + c_mf_m + d$, where $v$ is the value to predict, $c_i$ are constants, $f_i$ are feature values, and $d$ is the intercept. The standard approach to find such a function is to minimize the mean squared error, which is easy to compute. LASSO is a popular regularized version of this simple estimation that penalizes the absolute values of the regression coefficients $c_1, \ldots, c_m$. Formally, given a dataset of features $f_i^d$ and target values $v^d$, with $1 \leq d \leq k$ where $k$ is the number of samples, it uses convex optimization in order to find a regression function that solves the following problem\footnote{This is the version implemented in the scikitlearn Python package~\cite{scikit}, which we use to learn the models.}:
\[
\min \sum_{1 \leq d \leq k} \frac{1}{2 \cdot k} (v(f_1^d,\ldots,f_m^d) - v^d)^2 + \alpha \cdot \sum_{1 \leq i \leq m} |c_i|
\]
where $0 \leq \alpha \leq 1$ is a parameter for the effect of the regularization. Intuitively, the larger $\alpha$, the larger the penalty of having large coefficients $c_i$. Consequently, a larger value of $\alpha$ will drive more coefficients to zero. LASSO is a useful method when there are several correlated feature values, which could make an ordinary least squares model overfit on these values. We use LASSO regression because more zero coefficients implies we need to compute less feature values in order to evaluate the learned model, which has a positive effect on the optimization performance that we will discuss in Section~\ref{sec:translate}.

Currently, we provide the following features for these two regression models:
\begin{description}
\item[Feature 1: \texttt{sold}] For every item type $r$, the amount of $r$ items already auctioned.
\item[Feature 2: \texttt{remain}] For every item type $r$, the amount of $r$ items still to be auctioned.
 \item[Feature 3: \texttt{diff}] For every pair of item types $r$ and $r'$, the difference between the amount of $r$ and $r'$ items already auctioned.
\item[Feature 4: \texttt{sum}] For every item type $r$, the amount of value obtained from auctioning $r$ items, and the overall sum.
\item[Feature 5: \texttt{index}] For every item, the index at which it was auctioned.\end{description}

Other sequential features such as sliding windows and N-grams can of course be added to the model. However, since our white-box method computes these values inside an ILP solver, the only requirement is that they can be represented using an integer linear formulation. Although the \texttt{diff} feature can be determined using the first, we add it for convenience of learning a regression tree, which requires many nodes to represent such values. The influence of budget constraints is only directly modeled by the fourth feature: once the amount paid for $r_1$ items reaches a certain (to be learned) bound, we can expect all agents that only want $r_1$ items to be out of budget. Although, there only exists an indirect relation between the budget constraints and the first three features, including them can be beneficial and these are easier to compute. If used by the regression model, these features thus reduce the time needed to solve the auction design problem. For similar reasons, we add the last feature.
Below we give an example of how an ordering and its obtained values is transformed into a data set using these 5 types of features.
\begin{table}
\caption{The data set created from the past two auctions $\{r_1,r_2\}$ and $\{r_2,r_1\}$ in Example~\ref{ex:table}. \label{table:dataset}}
{\footnotesize
\begin{center}
\begin{tabular}{cccccccccc}
type & value & \texttt{sold}$r_1$  & \texttt{sold}$r_2$ & \texttt{diff}$r_1r_2$  & \texttt{sum}$r_1$ & \texttt{sum}$r_2$ & \texttt{sum}& \texttt{index} \\
$r_1$ & 4 & 0 & 0 & 0  & 0 & 0 & 0 &1\\
$r_2$ & 4 & 1 & 0 & 1 & 4 & 0 & 4& 2\\
$r_2$ & 5 & 0 & 0 & 0  & 0 & 0 & 0& 1\\
$r_1$ & 0& 0 & 1 & -1 & 0 & 5 & 5& 2\\
\end{tabular}
\end{center}
}
\end{table}
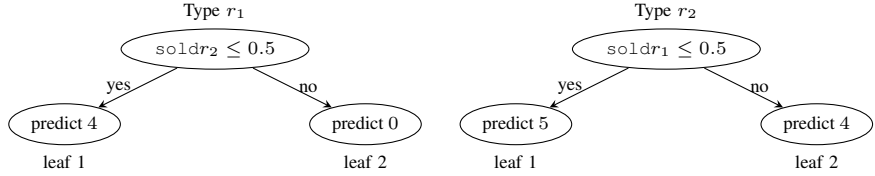
\begin{figure}[t]
\centering
{\scriptsize
\begin{tikzpicture}[>=stealth]
    \node [ellipse,draw,text centered] (root) at (0,0) {$\texttt{sold} r_2 \leq 0.5$};
    \node [ellipse,draw=none,text width=4em, text centered] at (0,0.5)  {Type $r_1$};
    \node [ellipse,draw,text centered] (left) at (-2,-1)  {predict $4$};
    \node [ellipse,draw=none,text width=4em, text centered] at (-2,-1.5)  {leaf 1};
    \node [ellipse,draw,text centered] (right) at (2,-1)  {predict $0$};
    \node [ellipse,draw=none,text width=4em, text centered] at (2,-1.5)  {leaf 2};
   \draw[->] (root) edge node [draw=none,left] {yes} (left);
   \draw[->] (root) edge node [draw=none,right] {no} (right);


    \node [ellipse,draw,text centered] (root) at (6,0) {$\texttt{sold} r_1 \leq 0.5$};
    \node [ellipse,draw=none,text width=4em, text centered] at (6,0.5)  {Type $r_2$};
    \node [ellipse,draw,text centered] (left) at (4,-1)  {predict $5$};
    \node [ellipse,draw=none,text width=4em, text centered] at (4,-1.5)  {leaf 1};
    \node [ellipse,draw,text centered] (right) at (8,-1)  {predict $4$};
    \node [ellipse,draw=none,text width=4em, text centered] at (8,-1.5)  {leaf 2};
   \draw[->] (root) edge node [draw=none,left] {yes} (left);
   \draw[->] (root) edge node [draw=none,right] {no} (right);

\end{tikzpicture}
}
\caption{\label{fig:trees} Two regression trees for the two item types from Example~\ref{ex:table}.}
\end{figure}

\begin{example}
\label{ex:table}
Consider the setting of Example~\ref{exp1}. 
Assume two auctions have been carried out. One sold $r_1$ first and then $r_2$. The other reversed. As shown in Example~\ref{exp1}, the first auction would obtain a revenue of 8, and the second auction would receive 5.  We compute feature values from these two auctions as depicted in Table~\ref{table:dataset}. Subsequently, we learn regression trees for both item types $r_1$ and $r_2$, as shown in Figure~\ref{fig:trees}.\footnote{The learned linear regression model is more straightforward. Hence we skip such an example here.}

After learning these regression trees, we can optimize the ordering for a new (unseen) multiset of items $\{r_1,r_2,r_2,r_2\}$ by trying all orderings and choosing one with maximum expected revenue: $r_1r_2r_2r_2$ gives $4+4+4+4=16$, $r_2r_1r_2r_2$ gives $5+0+4+4=13$, $r_2r_2r_1r_2$ gives $5 + 5 + 0 + 4 = 14$, and $r_2r_2r_2r_1$ returns $5+5+5+0=15$. Hence, the optimizer will choose to schedule the $r_1$ item before all $r_2$ items. In general, trying all possible orderings will be impossible: for a multiset of items $S = \{r_1,\ldots,r_n\}$ of $n$ types, there are a total of $\frac{|S|!}{\prod_{1 \leq i \leq n} |\{r_i | r_i \in S\}|!}$ unique orderings, which blows up very quickly.\QED

 \end{example}

\subsection{Modeling power and trade-off}\label{sec:power}

Our method of regression modeling allows the use of any regression method from machine learning for predicting unknown quantities in optimization, such as objective values and parameters. In addition, since the regression function $R$ uses other values in a (proposed) solution as input ($J,L$) instead of only external parameters/data, a learned regression model represents unknown relations between the different values in a solution. The model thus answers the question ``What is the value of X given that we do Y?'', as opposed to
``What is the value of X?'' that is answered by fitting only model parameters. Answering the first question allows for many more interesting possibilities. For instance, one could use stochastic optimization with fitted parameters to produce a schedule, use regression models to predict the effect of this schedule on the parameters, and use stochastic optimization again on the newly estimated parameters.
This way, one can use machine learning tools to plan further ahead. Using our white-box method, this can even be done using a single call to the optimization software.

This loop-back functionality provides a lot of power to our method, but also comes with a risk. Every time the predictive models are used there is a probability that the predictions are inaccurate. When using a loop-back, these possible inaccuracies influence all future predictions that depend on it. These future predictions are thus more inaccurate and the predicted overall objective value can potentially diverge from the true value. These cascading inaccuracies are an issue, however, the added modeling power makes up for it. We make use of it in the \texttt{sum} feature, which relates the predicted value to the predictions of earlier auctioned items. This feature is very important for predicting budget constraints, and consequently is often used by the regression models to produce predictions.

\section{White-box and black-box optimization for OOSA}\label{sec:translate}


Given the predictive models for the expected value per item, it is not straightforward to compute a good ordering as we already showed in Example~\ref{ex:table}. For a given ordering, we can predict the individual revenues of items using the regression model, and sum these up to obtain the revenue of the ordering. However, testing all possible orderings and choosing the one with the highest revenue will take a very long time. For instance, when we want to order 40 items of 8 types (the experimental setting in Section~\ref{sec:exp}) with 5 of every type, we will need to test $\frac{40!}{5!^{8}} \approx 1.9 \cdot 10^{31}$ possible unique orderings.


In~\ref{sec:proofs}, we also provide hardness results that demonstrate there is little hope (unless $\P=\NP$) of finding an efficient (polynomial time) algorithm that gives the optimal ordering for any regression tree or linear regression predictor. In general, we cannot do better than performing a guided search through the space formed by all possible orderings. We present two such search-based optimization methods: (1) a novel ``white-box'' optimization (i.e., ILP model), and (2) a ``black-box'' heuristic (i.e., best-first search).

\subsection{White-box optimization: an ILP model}

Given regression tree and linear regression models for the expected value per item type, we automatically formulate the problem of finding an optimal ordering as a mixed integer linear program (MIP). We discuss the encoding of a sequential auction, feature values, objective function, and translating the learned models (regression tree and linear regression respectively) below.

\paragraph{Ordering an auction}
Given a multiset $I$ of $n$ items, each from a set of possible types $R$, we use the following free variables to encode any possible ordering of $I$: {$x_{i,r} \in \{0,1\}$}. Item $i$ is of type $r$ if and only if $x_{i,r} = 1$.
Thus, if $x_{3,A}$ is equal to 1, it means that the third auctioned item is of type $A$. We require that at every index $i$ at most one item type is auctioned, and that the total number of auctioned items of type $r$ is equal to the number $n_r$ of type $r$ items in $I$.

\begin{center}
\begin{tabular}{rlll}
$\sum_{r \in R} x_{i,r}$  & = & $1$ & for all $1 \leq i \leq n$\\
$\sum_{1 \leq i \leq n} x_{i,r}$ & =& $n_{r}$ & for all $r \in R$\\
\end{tabular}
\end{center}

Any assignment of ones and zeros to the $x$ variables that satisfies these two types of constraints corresponds to a valid ordering of the items in $I$. The value of such an ordering is determined by the learned regression models.

\paragraph{Translating feature values}
In order to compute the prediction of a regression model, we not only need to translate the models into ILP constraints, but also the values of the features used by these models. Feature 1, 2, 3, and 5 can be computed using linear functions from the $x$ variables:

\begin{center}
\begin{tabular}{rlll}
$\texttt{sold}_{i,r}$ & = & $\sum_{j<i}x_{j,r}$ & for all $1 \leq i \leq n, r \in R$\tabularnewline
$\texttt{diff}_{i,r,r'}$ & = & $\texttt{sold}_{i,r} - \texttt{sold}_{i,r'}$ & for all $1 \leq i \leq n, r,r' \in R, r \ne r'$\tabularnewline
$\texttt{remain}_{i,r}$ & = & $\sum_{j>i}x_{j,r}$ & for all $1 \leq i \leq n, r \in R$\tabularnewline
$\texttt{index}_{i}$ & = & $i$ & for all $1 \leq i \leq n$\tabularnewline
\tabularnewline
\end{tabular}
\end{center}

For the fourth type of feature, we use an additional variable  $v_{j,r}$, which encodes the expected value of the item auctioned at index $j$ of type $r$. If the item at index $j$ is not of type $r$, $v_{j,r}$ is equal to $0$. Since the $v$ variables require the predictions (expected values) as input, we provide their definition after defining the regression models.

\begin{center}
\begin{tabular}{rlll}
$\texttt{sum}_{i,r}$ &  =  & $\sum_{1 \leq j \leq i} v_{j,r}$ &  for all $1 \leq i \leq n$, $r \in R$
\end{tabular}
\end{center}

To aid the ILP solver, we also pre-compute the minimum $m_{f,i}$ and maximum $M_{f,i}$ obtainable values of every feature $f$ at every index $i$ and provide these as bounds to the solver.

\paragraph{Constructing the objective function}
We aim to maximize the expected values $v_{i,r}$: 

\begin{center}
\begin{tabular}{c}
$\max \sum_{1 \leq i \leq n} \sum_{r \in R} v_{i,r}$
\end{tabular}
\end{center}

These $v_{i,r}$ are the predictions of the regression functions, defined later. 
Although it is also possible to compute both the objective function and the $\texttt{sum}$ values as very large sums over the $x$ and model variables (described next), specifying parts of these sums as intermediate continuous $v$ variables significantly reduces both the encoding size and the computation time.
%


Finally, we discuss how to encode the learned regression tree and the linear regression model as the constraints in the ILP model.

\paragraph{Encoding regression trees}

We translate the regression tree models into ILP using carefully constructed linear functions. Our encoding only requires one new set of $\{0,1\}$ variables $z_{i,l,r}$, representing whether a leaf node $l$ is reached for item type $r$ at index $i$. The internal (decision) nodes of the trees can be represented implicitly by the constraints on these new $z$ variables. Intuitively, we encode that a $z$ variable has to be false when the binary test of any of its parent nodes fails. By additionally requiring that exactly one $z$ variable is true at every index, we fully encode the learned regression trees.


Let $D_r$ be the set of all decision nodes in the regression tree for type $r$. Every decision node in $D_r$ contains a boolean constraint $f \leq c$, which is true if and only if feature $f$ has a value less than or equal to a constant $c$. A key insight of our encoding is that every such boolean constraint directly influences the value of several $z$ variables: if it is true (at index $i$), then all $z$ variables representing leafs in the right subtree are false; if it is false, then all that represent leafs in the left subtree are false. In this way, we require only two constraints per boolean constraint in order to represent all possible paths to leaf nodes.

\begin{center}
\begin{tabular}{cccc}
$\texttt{fv}_{f,i} + (M_{f,i} - c) \cdot \sum_{l \in L} z_{i,l,r} \leq M_{f,i}$ & for all $1 \leq i \leq n, r \in R, (f \leq c) \in D_r$\\
$\texttt{fv}_{f,i} + (m_{f,i} - c) \cdot \sum_{l \in L'} z_{i,l,r} \geq m_{f,i}$ & for all $1 \leq i \leq n, r \in R, (f \leq c) \in D_r$
\end{tabular}
\end{center}

\noindent where $\texttt{fv}_{f,i}$ is a calculation of feature value $f$ for index $i$, $L$ and $L'$ are the leaf nodes in the left and right subtrees of the decision node with constraint $(f \leq c)$ in the regression tree for type $r$, and $M_{f,i}$ and $m_{f,i}$ are the  maximum and minimum values of feature $f$ at index $i$. For the feature calculation we simply replace $\texttt{fv}_{f,i}$ with the right-hand sides of the corresponding feature definitions.\footnote{We ignore the possibility that a feature $f$ is equal to $c$ because, since features have a limited precision, we can always replace the constants in a decision node with one that cannot be obtained by $f$, without changing its behavior.}

The above constraints ensure that when $z_{i,l,r}$ obtains a value of $1$, all of the binary test in the parent nodes on the path to $l$ in the tree for type $r$ return true at index $i$. By construction of the regression trees, this ensures that at most one $z$ variable is true for every type $r$ and index $i$. We require however that exactly one $z$ variable is true at every index.\footnote{Counterintuitively, it can occur that the objective function (discussed below) is maximized when every $z$ variable is false at an index $i$. If a small $\texttt{sum}$ is needed to reach a high revenue prediction, it can be beneficial to auction but not sell an item.} This $z$ has to be of the same type as the item sold at index $i$, denoted by the $x$ variables:

\begin{center}
\begin{tabular}{rlll}
$\sum_{l} z_{i,l,r}$ & $=$ & $x_{i,r}$ & for all $1 \leq i \leq n, r \in R$\\
\end{tabular}
\end{center}

This completes our encoding of the regression trees. The predictions of the trees at every index $i$ are given by the $z$ variable that is true for index $i$. We multiply this $z$ variable with the constant prediction in the leaf node it represents to obtain the prediction, and store it in the $v$ variables used to compute the $\texttt{sum}$ feature values.

\begin{center}
\begin{tabular}{rlll}
$v_{i,r} = \sum_{l \in L_r} c_{l,r} \cdot z_{i,l,r}$, & for all $1 \leq i \leq n, r \in R$ \\
\end{tabular}
\end{center}

\noindent where $c_{l,r}$ is the constant prediction of leaf $l$ in the tree for type $r$.

We now discuss how to encode a linear regression model.

\paragraph{Encoding linear regression model}

Due to its linear nature, implementing linear regression in ILP is very straightforward. We can directly compute the value of the $v$ variables using the linear predictor function:

\begin{center}
\begin{tabular}{rlll}
$v_{i,r} = \sum_{f \in \texttt{Feat}} c_{f,r} \cdot \texttt{fv}_{f,i}$, & for all $1 \leq i \leq n, r \in R$ \\
\end{tabular}
\end{center}

\noindent where $\texttt{Feat}$ is the set of all features, $\texttt{fv}_{f,i}$ is feature $f$'s values at index $i$, and $c_{f,r}$ is the constant coefficient for feature $f$ in the regression function for type $r$. The only somewhat difficult part is that at every index, the used regression function can change depending on the auctioned item type $r$. We implemented this choice using indicator functions in CPLEX.\footnote{Many other solvers have similar constructions. If not, these constraints can be implemented using a `big-M' formulation, similar to the one we use to determine the value of the $z$ variables in the regression tree formulation.} This changes the above formulation as follows:

\begin{center}
\begin{tabular}{rrrlll}
$x_{i,r} = 1$ & $\rightarrow$ & $v_{i,r} = \sum_{f \in \texttt{Feat}} c_{f,r} \cdot \texttt{feat}_{f,i}$, & for all $1 \leq i \leq n, r \in R$ \\
$x_{i,r} = 0$ & $\rightarrow$ & $v_{i,r} = 0$, & for all $1 \leq i \leq n, r \in R$ \\
\end{tabular}
\end{center}

It states that if $x_{i,r}$ is true, then the values of $v_{i,r}$ is determined using the regression function. Otherwise, its value is 0. These are the only constraints needed to fully implement a linear regression function. When using LASSO regularization, some coefficients will receive the value 0. These are removed from the encoding, making the models smaller and easier to evaluate in the solver.

Now the two ILP models are complete and ready to solve the OOSA problem. We give the following example to illustrate how the formulation of ILP works given learned regression trees.

\begin{example}
Given the learned trees in Example~\ref{ex:table}, suppose we are asked to order a new multiset of items $\{r_1,r_2,r_2\}$. We translate this new set, together with the learned trees into the following integer linear program with the following $\{0,1\}$ decision variables (for all $1 \leq i \leq 3$): $x_{i,r_1},x_{i,r_2}, z_{i,1,r_1},z_{i,2,r_1},z_{i,1,r_2}$:

\begin{center}
\begin{tabular}{c}
$\max \sum_{1 \leq i \leq 3} v_{i,r_1} + v_{i,r_2}$
\end{tabular}\\
\begin{tabular}{rcl}
where $v_{i,r_1}$ & $=$ & $4 z_{i,1,r_1}$ \\
and $v_{i,r_2}$ & $=$ & $5 z_{i,1,r_2} + 4 z_{i,2,r_2}$
\end{tabular}
\end{center}

\noindent subject to (for all $1 \leq i \leq 3$)
\begin{center}
\begin{tabular}{ccccccccccccccc}
$x_{1,r_1} + x_{2,r_1} + x_{3,r_1}$ & $=$ & $1$ \\
$x_{1,r_2} + x_{2,r_2} + x_{3,r_2}$ & $=$ & $2$ \\
$x_{i,r_1} + x_{i,r_2}$ & $=$ & $1$
\end{tabular}
\end{center}
This denotes that exactly one $x$ variable is true at every index $i$, $2$ $x$ variables are true for item type $r_2$, and $1$ for type $r_1$. This encodes all possible orderings. From this we compute the feature values (for all $1 \leq i \leq 3$):
\begin{center}
\begin{tabular}{ccccccccccccccc}
$\texttt{sold}_{i,r_1}$ & $=$ & $x_{1,r_1} + \ldots + x_{i-1,r_1}$\\
$\texttt{sold}_{i,r_2}$ & $=$ & $x_{1,r_2} + \ldots + x_{i-1,r_2}$\\
\end{tabular}
\end{center}
that are used in the constraints denoting the Boolean tests in the internal nodes:
\begin{center}
\begin{tabular}{ccccccccccccccccc}
$\texttt{sold}_{i,r_2}$ & $+$ & $(100-0.5)z_{i,1,r_1}$ & $\leq$ & $100 $ \\
$\texttt{sold}_{i,r_2}$ & $+$ & $(-0.5)z_{i,2,r_1}$ & $\geq$ & $0$\\
$\texttt{sold}_{i,r_1}$ & $+$ & $(100-0.5)z_{i,1,r_2}$ & $\leq$ & $100$ \\
$\texttt{sold}_{i,r_1}$ & $+$ & $(-0.5)z_{i,2,r_2}$ & $\geq$ & $0$\\
\end{tabular}
\end{center}
where $M_{\texttt{sold},i}=100$ and $m_{\texttt{sold},i}=0$. The first two constraints encode that if $z_{i,1,r_1} = 1$, $\texttt{sold}_{i,r_2} \leq 0.5$; and if $z_{i,2,r_1} = 1$, $\texttt{sold}_{i,r_2} \geq 0.5$. Thus, if a $z$ variable is true for a leaf, then all the Boolean tests of internal nodes on the path from the root to that leaf have to succeed. The last two constraints are the same constraints for the second tree. At last, we require that exactly one $z$ variable is true at every index:
\begin{center}
\begin{tabular}{ccccccccccccccccc}
$z_{i,1,r_1} + z_{i,2,r_1}$ & $=$ & $x_{i,r_1}$, \\ $z_{i,1,r_2} + z_{i,2,r_2}$ & $=$ & $x_{i,r_2}$.
\end{tabular}
\end{center}

\noindent A satisfying assignment to the $x$ variables is $x_{1,r_1},x_{2,r_2},x_{3,r_2}$ set to 1, the rest to 0, corresponding to the ordering $r_1r_2r_2$. Since $\texttt{sold}_{1,r_2} = 0$, this leads to $99.5z_{1,1,r_1}\leq 100$ and $-0.5 z_{1,2,r_1}\geq 0$, forcing $z_{1,2,r_1}=0$. Since $z_{1,1,r_1} + z_{1,2,r_1} = x_{1,r_1} = 1$, this implies $z_{1,1,r_1}=1$. For the next index, we have $\texttt{sold}_{2,r_1} = 1$, giving $99.5z_{2,1,r_2}\leq 99$ and $-0.5 z_{2,2,r_2}\geq -1$, and therefore forcing $z_{2,1,r_2}=0$ and $z_{2,2,r_2}=1$. Similarly, we obtain $z_{3,2,r_2}=1$. This results in an objective value of $4 + 4 + 4 = 12$.

\QED
\end{example}

\subsection{A black-box heuristic: best-first search algorithm}

We also provide a black-box heuristic for solving the ordering problem, see also~\cite{verwer2012revenue}. The traditional method to overcome the computational blowup caused by sequential decision making is to use a dynamic programming method. Although this lessens the computational load by combining the different paths that lead to the same sets of auctioned items, the search space is still too large and waiting for a solution will take too long. Instead, we therefore employ a best-first search strategy that can be terminated anytime in order to return the best found solution so far. We show how this best-first search strategy works in Algorithm~\ref{alg:ordering}.

\begin{algorithm}[t]
\begin{small}
\caption{Black-box heuristic for solving OOSA: best-first search  \label{alg:ordering}}
\begin{algorithmic}
\REQUIRE A set of items $S$, historical data on orderings and their values $D$, a maximum number of iterations $m$
\ENSURE Returned is a good (high expected value) ordering
\STATE Transform $D$ into a data set
\FOR{ every item type $r$ }
\STATE Learn a regression model from $D$ for predicting the value of item type $r$
\ENDFOR
\STATE Initialize a hashtable $H$ and a priority queue $Q$
\STATE Add the empty data row to $Q$
\WHILE{ $Q$ is not empty and the size of $H$ is less than $m$ }
\STATE Pop the row of features $F$ with highest value $v$ from $Q$
\IF{ $H$ does not contain $F$ with a value $\geq v$ }
\STATE Add $F$ with value $v$ to $H$
\STATE Let $L$ be the set of remaining items in $F$
\FOR{ every item type $r$ of items in $L$ }
\STATE Let $i_k$ be an item of Type $r$ in $L$
\STATE Let $L'$ be a random ordering of $L - i_k$
\STATE Use the models to evaluate the value $v'$ of auctioning the ordering $i_k L'$ after $F$
\STATE Create new features $F'$ for auctioning $i_k$ after $F$
\STATE Add $F'$ to $Q$ with value $v+v'$
\ENDFOR
\ENDIF
\ENDWHILE
\RETURN The highest evaluated ordering
\end{algorithmic}
\end{small}
\end{algorithm}

The algorithm uses a hashtable and a priority queue. The hashtable is used to exclude the possibility of visiting the same nodes twice if the obtained value is less than before (just like a dynamic programming method). These dynamic programming cuts are sensible but lose optimality as on rare occasions it could be better to sell earlier items for less, leaving more budget for the remaining ones. The priority queue provides promising candidate nodes for the best-first strategy. By computing random orderings of the remaining items, the learned models can evaluate complete orderings of all items. The best one found is stored and returned if the algorithm is terminated. Unfortunately, this does not result in an admissible heuristic for an A* search procedure. Hence, even if the algorithm pops a solution from the queue, this is not necessarily optimal. In our experience, using random orderings of the remaining items in this heuristic provides a good spread over the search space. Although some nodes can be `unlucky' and obtain a bad ordering of the remaining items, there are always multiple ways to reach nodes in the search space and it is very unlikely that all possibilities will be `unlucky'.

\subsection{White-box or black-box optimization?}

The main difference between the two abovementioned approaches (see Figure~\ref{fig:white-black}) is that the white-box method specifies the predictors entirely as constraints, which can be used to infer bounds on the predictions and cut the search space. The black-box method instead uses the predictors as oracles and is ignorant as to how the predictions are made, which are naturally more efficient to compute but cannot be used to infer search space cuts, i.e., to deduce that one ordering is better than another without testing both of them. Another key difference is that the white-box method results in a single optimization model that can be run in any modern solver, while the black-box method requires the use of executable code to produce the predictions. In the black-box setting, it is therefore much harder to use the powerful solving methods available in dedicated solvers for problems such as integer linear optimization (ILP), satisfiability (SAT), or constraint programming (CP). 
Instead, general search methods can be used such as best-first search, beam-search, meta-heuristics, genetic algorithms, 
etc.

Both black-box and white-box approaches have their advantages. The main advantage of black-box is that its performance is for a large part independent of the complexity of the used regression model. In contrast, by explicitly modeling the regression model as constraints in a white-box, more complex regression models lead to many more constraints, which can dramatically increase in the time needed to solve it. Another advantage of black-box optimization is that it is easy to include additional cuts such as the dynamic programming cuts discussed above. Such cuts can be added as constraints in an LP formulation, but this can lead to a blow-up in runtime.

The main benefit of using the white-box approach is the use of modern exact solvers instead of a heuristic search. These solvers use (amongst others) advanced branch-and-bound methods to cut the search space, compute and optimize a dual solution, and can prove optimality without testing every possible solution. Since
Our white-box constructions can also be easily integrated into existing (I)LP formulations that have been used in a wide range of applications in for instance Operations Research. In this way, one can combine the vast amount of expert knowledge available in these applications with the knowledge in the readily available data. Our white-box method is the first we know of that makes the results of machine learning directly available to mathematical modeling in this fashion.

The most important downside of white-box is that an evaluation of translated models likely requires more time than running the code as a black-box, especially when the models or features are somewhat complex. In our opinion, however, the advantages of white-box optimization largely outweigh those of black-box optimization and make it a very interesting topic for research in machine learning and optimization. In the next section, we compare these two modeling approaches and investigate this trade-off by applying them to OOSA.




%

\section{Experiments}\label{sec:exp}




Designing an optimal ordering for sequential auctions is difficult with heterogeneous bidders, as they may value items differently, have different budget constraints, and moreover, bid rationally or irrationally with various bidding strategies. To evaluate the performance of the proposed optimization methods, ideally, we should collect real auction data, build the optimization models, run real-world auctions with real bidders using different ordering of items produced by different methods, and then compare the resulting revenues. Since this evaluation method is not feasible for us, nor is it the main purpose of this paper, we opted for a widely accepted evaluation approach in research community, that is, we created an auction simulator which simulates auctions with agents. We used this simulator to generate auction data sets, and to evaluate the proposed method. An overview of this process is given in Figure~\ref{fig:experimentframework}.


\begin{figure}[thb]
\begin{center}
\includegraphics[width=\columnwidth]{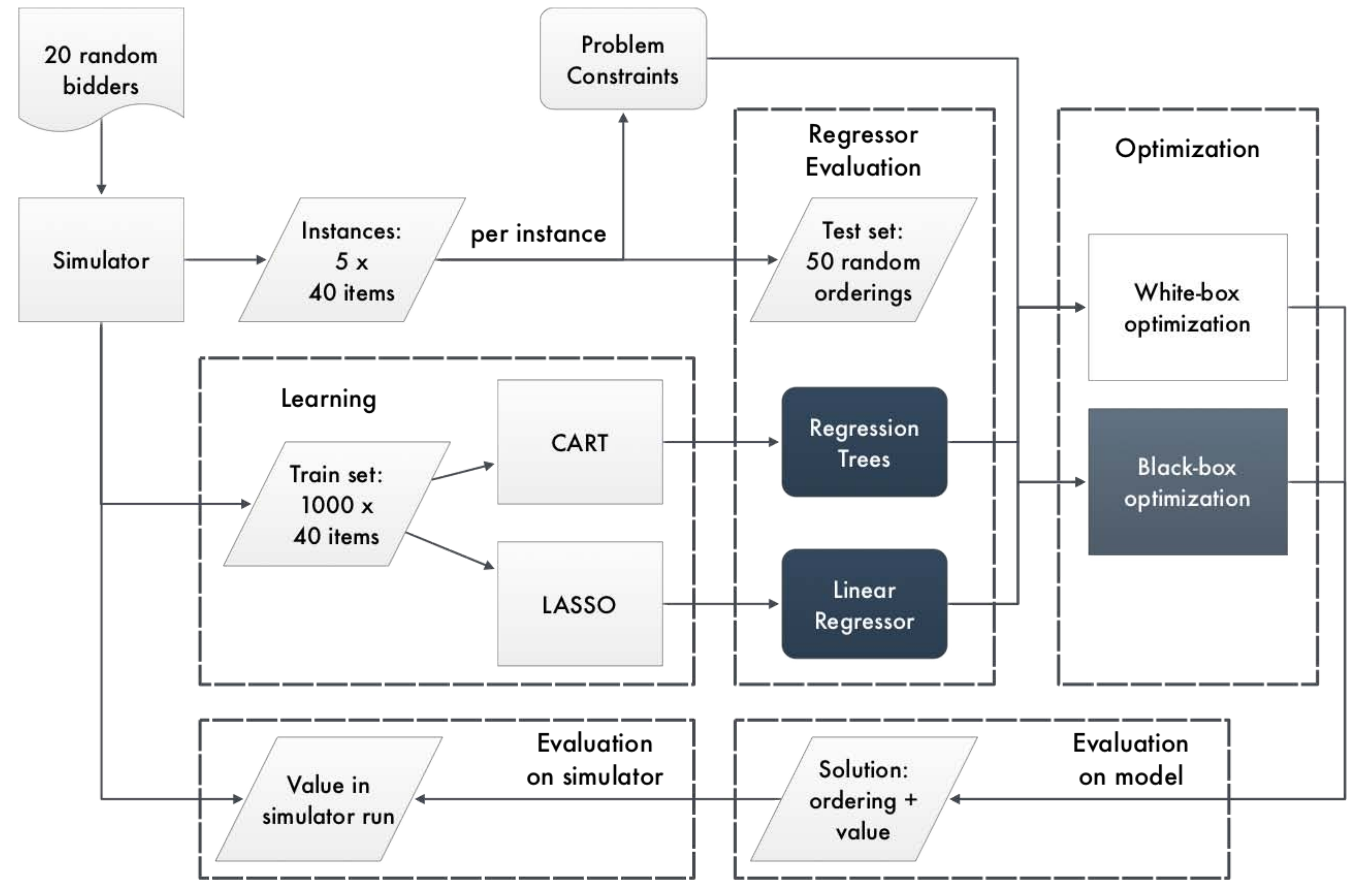}
\end{center}
\caption{Our framework for evaluating optimization using learned models for OOSA. There is a simulator that is used in two ways: to generate historic data and to evaluate the OOSA solutions. A train set and a few unseen problem instances are generated. The train set is used to learn the regression models. Random orderings of the problem instances are used to test them. The instances together with the learned models are provided as input to the ILP-based white-box optimization and the best-first black-box optimization. The resulting orderings are evaluated using the learned models, and using simulator runs.\label{fig:experimentframework}}
\end{figure}

\subsection{The simulator}

\paragraph{Simulating auctions} We simulate several sequential auction settings with the simulator: (i) first price auctions where agents bid their reservation prices (agents' reservation prices are lower than their valuations), as in~\cite{Subramaniam09optimal}; (ii) second price sealed bid auctions (i.e. Vickrey auctions~\cite{vickrey61}), where agents bid truthfully on each item based on their valuation functions, as in \cite{Pinker10} \footnote{Note that in sequential Vickrey auctions with budget constrained agents, truth-telling is not an equilibrium bidding strategy (see~\cite{Vetsikas13}). To the best of our knowledge, how to compute the equilibrium bidding strategy in realistic auction settings is an open problem.};  (iii) Vickrey auctions where agents bid smartly, i.e., they compare the utility obtained at the end of the auction when buying and not buying the item and place a bid based on the difference (see Section~\ref{sec:exp23} for more details). On the last auctioned item, they bid truthfully\footnote{Vickrey~\cite{vickrey61} showed it is a weakly dominant strategy for bidders to bid their true values for the last auctioned item.}.
Given all bids on an item, the highest bid wins. If multiple agents have the same true value, one of these is selected as winner uniformly at random. 
With these different auction settings, we intend to show our method is robust to the actual auction format and bidding strategies.

\paragraph{Item types}
We use a given set of 8 items to initialize the auction simulator. Every type gets assigned an average value of $25 + (5 \cdot i)$, for $1 \leq i \leq 8$, and a reserve price of $\frac{1}{2}(25 + (5 \cdot i))$. Every type is assigned popularity and sparsity values, drawn uniformly from $[2,10]$. The popularity value measures the degree of desirability of the item type by the agents. The sparsity is a measure for the frequency that an item type is sold.
In every auction, $40$ items are generated using a roulette wheel drawing scheme using the sparsity values.

\paragraph{Bidder agents}
The simulator starts with 20 randomly generated bidding agents. Every such agent gets assigned a budget between 50 and 250 uniformly at random. They may desire 1 to 5 of the 8 item types, where popular types have a higher probability of being selected, drawn using a roulette wheel selection on the item types' popularity values. For bidding in first-price auctions, every desired item type assigned to an agent is also given a reservation price of $u \cdot v$, where $v$ is the average value of that type, and $u$ is a uniform random value between 0.5 and 2.0. The following is an example of five agents generated for the experiments:
\begin{center}
{\footnotesize
\begin{tabular}{c|cccccccc}
budget & $v(r_1)$ & $v(r_2)$ & $v(r_3)$ & $v(r_4)$ & $v(r_5)$ & $v(r_6)$ & $v(r_7)$ & $v(r_8)$ \\ \hline
123 & & & & & 59 & & & \\
141 & & & & & 78 & 81 & 41 & 115 \\
130 & & 32 & 21 & & 80 & 25 & & 36 \\
126 & & 32 & 53 & & 28 & & 36 & \\
182 & & & & & 51 & 93 & & \\
\end{tabular}
}
\end{center}

\paragraph{Training}
For a given set of 20 random bidders, the simulator generates 1000 historical auctions. The 40 items in these auctions are generated using the above scheme, and ordered randomly. These items are run in the simulator, where the agents use the above-mentioned bidding strategies to decide what value to bid, in order to determine the winners and the item selling prices. The total selling price of all items in one sequential auction is the collected revenue. For the 20 generated bidders, we first experiment the effect of different item orderings on the collected revenue by trying 100 random orderings and comparing the smallest, median, and largest collected revenue. If the difference between the largest and smallest is less than one tenth of the median revenue, 20 new bidders are generated. This process is repeated until we find a set of agents that passes this check, which typically occurs after a few iterations. By performing this check, we remove irrelevant problem instances. For the 20 agents that pass this check, we generate 1000 auctions of 40 items and simulate these auctions together with the agents. The resulting sequences of item-price pairs are then transformed to the features discussed in Section~\ref{sec:representing}, and the resulting data set is used to train the regression trees and linear regressors.

\paragraph{Testing}
For the same set of 20 bidders, we generate 5 sets of 40 items, which are used for testing. First, the regressors are tested by comparing their predictions with the revenues generated by the simulator on 50 random orderings of each of these item sets. Second, we translate each of the item sets into constraints for both the black-box and white-box optimization solvers. The best ordering found by these solvers are compared based on their values on the regression model, and in the simulator.


\subsection{Experimental setup}
In each experiment, we generate agents and items as described above.
We use an implementation of regression trees and LASSO from the scikit-learn machine learning module \cite{scikit} in Python to learn (and evaluate) the regressors. We learn trees of different depths of 3, 5, 8 (we call them \texttt{tree3}, \texttt{tree5}, \texttt{tree8}), and we set the minimum number of samples required to split an internal node to 10. The LASSO regressor is run with 3 different values for $\alpha$: 1.0, 0.1, and 0.000001 (\texttt{lasso1}, \texttt{lasso2}, \texttt{lasso3}). The resulting trees and linear models then get translated to ILP, which in turn gets solved by an ILP-solver (CPLEX~\cite{cplex}). In addition, we provide the solver with an initial solution (the best of 1000 random orderings) in order to start the search, and set the focus of the solver to finding integer feasible solutions. 
We set a time limit on the ILP solver of 15 minutes for each instance using a single thread on an Intel core i-5 with 8GB RAM and record the best ordering of items that the ILP solver has obtained. The last minute is spent on solution polishing (a local search procedure in CPLEX). We apply our best-first search method on the same problem instances with the same running time limit.

\paragraph{Evaluations}
There are two levels of evaluations involved in our problem. Firstly, we determine the quality of the learned regressors, as they influence the quality of the solution after optimization. For this, we tested different regression trees with different maximum depths and linear regression models with different parameters.

Next, the optimization methods are evaluated in terms of the quality of the predicted revenue. The optimization methods that we compare include the proposed white-box ILP model which finds a solution based on the abovementioned 6 regression models, the proposed black-box best-first search which evaluates a solution based on the 6 regression models, and in addition, two other simple ordering methods: (i) auctioning the most valuable item first (i.e., \texttt{mvf}), as suggested in~\cite{Subramaniam09optimal}; and (ii) a random ordering strategy (i.e., \texttt{mean5000}), as seen in many real-world auctions for the purpose of fairness.

It is not feasible for us to compute the best solution given the problem size. Thus, we obtain a lower bound on the optimal solution as follows: given a set of items, we generate 5000 random orderings, and we use the \emph{true} model (i.e., the simulator) to evaluate them and pick the one returning the highest revenue.
We use the mean value of these 5000 random orderings as the output of the random ordering strategy.

We evaluate the 15 ordering methods in two ways:
\begin{itemize}
\item Model evaluation: we use the learned regression models to evaluate the solutions returned by the ordering methods to compute the predicted revenues. 
\item Actual evaluation: we run auctions with the solutions (i.e., orderings) returned by the ordering methods in our simulator to obtain the corresponding revenues. Note that such an evaluation is possible only when a simulator is available.
\end{itemize}

There are in total three sets of experiments presented in this paper.
\begin{enumerate}
\item In the first set of experiments, we simulate first price auctions where agents bid and pay (if winning) their reservation prices. 

\item In the second set of experiments, the simulator runs Vickrey auctions where agents bid their true values on each item and the winner pays the second highest bid, or the reserve price if (s)he is the only bidder.

\item In the third set of experiments, the simulator runs Vickrey auctions and agents bid smartly based on their expected utility at the end of the auction (see~\ref{sec:exp23}).

\end{enumerate}

\begin{figure}[thb]
\begin{center}
\includegraphics[width=\columnwidth]{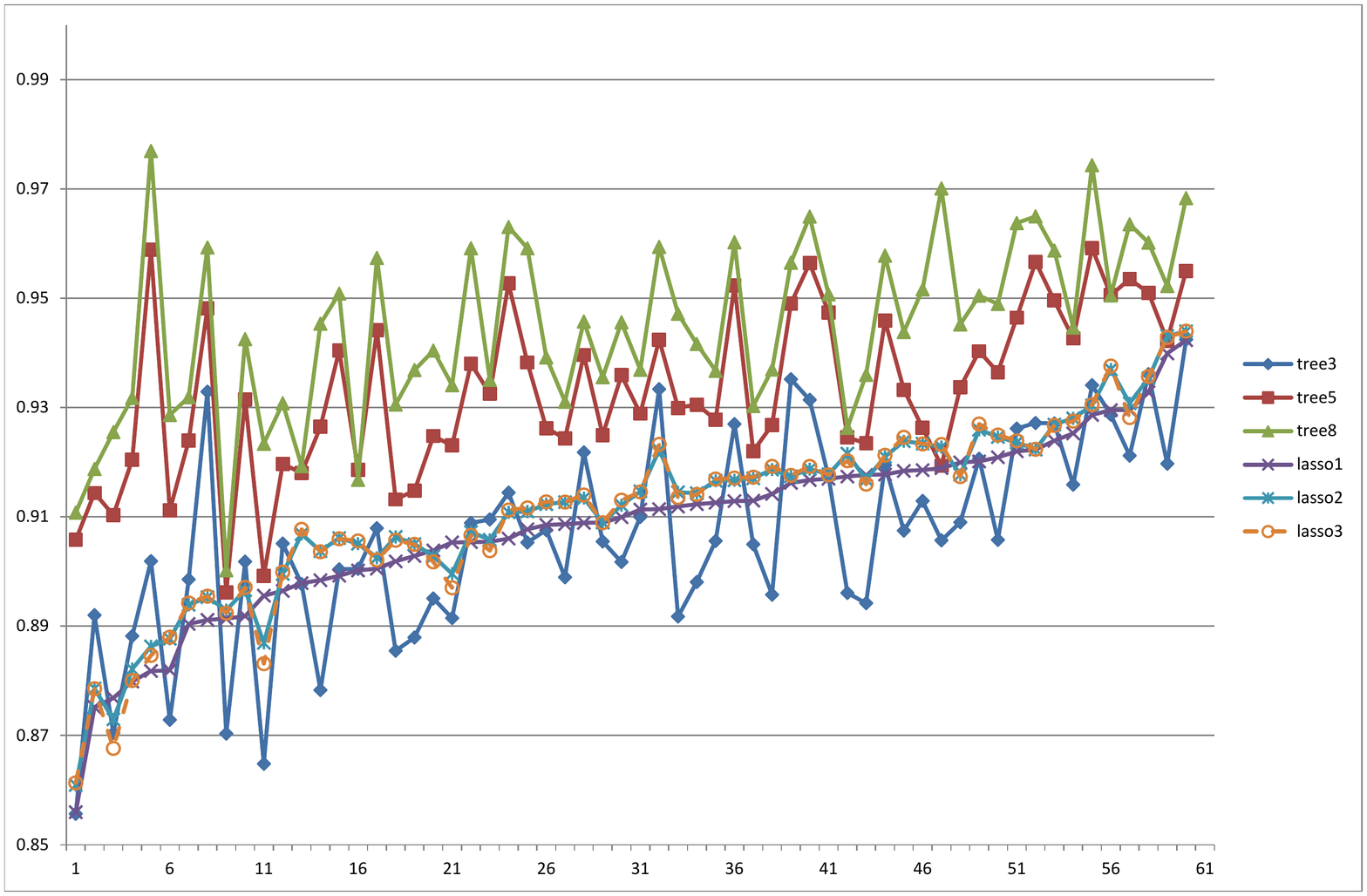}
\end{center}
\caption{$R^2$ scores for different learning models for 60 different sets of agents. The observed values are those returned by the simulator. Each score is computed from 10000 values. The results are presented in ascending order of the scores of \texttt{lasso1}. \label{fig:r2}}
\end{figure}

\subsection{Experiment 1}

\begin{table*}[tbh]
\begin{center}
\caption{The frequencies of wins of 300 runs for each method against others (row method vs column method), evaluated in the simulator.\label{tab:wins}}
\resizebox{\columnwidth}{!}{
\begin{tabular}{l|ccccccccccccccccc}
   & \begin{sideways}    tree3lp \end{sideways}& \begin{sideways}tree3bf \end{sideways}&\begin{sideways}tree5lp\end{sideways} &\begin{sideways}tree5bf\end{sideways}& \begin{sideways}tree8lp\end{sideways} &\begin{sideways}tree8bf\end{sideways} &\begin{sideways}lasso1lp\end{sideways}&\begin{sideways} lasso1bf\end{sideways}& \begin{sideways}lasso2lp\end{sideways}&\begin{sideways} lasso2bf\end{sideways}&\begin{sideways} lasso3lp\end{sideways} &\begin{sideways}lasso3bf\end{sideways}&  \begin{sideways} mvf \end{sideways} &\begin{sideways}best5000 \end{sideways}& \begin{sideways}mean5000 \end{sideways}&\begin{sideways}total wins\end{sideways} \\\hline
 tree3lp  &  0 & 171 & 105 & 152 & 107 & 125 & 102&  130 &  71 &  127 &   82 &  119  &   231 &   17 &  240 & 2068\\
 tree3bf & 125  &  0  & 83 & 129 &  81 & 105 &  71 & 116 &  61  &  94 &   67 &   86  &  226 &   15 &  230 & 1770\\
 tree5lp & 192 & 213  &  0 & 197 & 149 & 168 & 144 & 169 & 111  & 156 &  128 &  164  &   249 &   24 &  273 & 2632\\
 tree5bf  &142 & 163  & 97 &   0 & 105 & 122 &  94 & 134 &  74  & 118 &   90 &  115  &    238 &   16 &  241 & 2037\\
 tree8lp  &188 & 217 & 142 & 193 &   0 & 178 & 137 & 181 & 113  & 160 &  128 &  162  &    260 &   23 &  287 & 2668\\
 tree8bf  &170 & 192 & 130 & 168 & 121 &   0 & 112 & 157 &  91  & 133  & 106 &  138  &    251 &   28 &  269 & 2361\\
 lasso1lp & 193 & 226 & 150 & 202 & 160 & 183 &   0 & 185&  113 &  178 &  136 &  170 &     263 &    38 &  269 & 2756\\
 lasso1bf & 168 & 181 & 127 & 158 & 113 & 139 & 100 &   0 &  88 &  135 &   91 &  131 &     233 &   18  & 251 & 2219\\
 lasso2lp  &224 & 236 & 184 & 225 & 182 & 205 & 174 & 204  &  0  & 195 &  159  & 199 &     265 &   40 &  286 & 3073\\
lasso2bf  &171  & 200 & 141 & 177 & 134 & 163 & 115 & 155  & 94  &   0 &  107  & 124 &     243 &   17 &  273 & 2407\\
lasso3lp  &214 & 231 & 169 & 204  &167  &188  &155 & 201 & 126  & 184  &   0  & 189  &    257  &  36 &  281 & 2899\\
lasso3bf & 176 & 210 & 131 & 182 & 134 & 161 & 124 & 154 &  89  & 156  &  96  &   0  &    245  &  15 &  276 & 2439\\
mvf  & 68  & 73 &  51  & 61 &  38  & 48 &  36 &  65  & 34  &  55  &  41  &  54  &     0  &   9 &  123  & 982\\
best5000 & 282 & 284 & 268 & 284 & 277 & 271 & 257 & 280  &255  & 278 &  252 &  280 &     291  &   0 &  299  &4158\\
mean5000 &  57 &  65  & 27  & 58  & 13 &  30  & 31  & 48   &12  &  26 &   18 &   24  &   174  &   1  &   0   &864\\
\end{tabular}
}
\end{center}
\end{table*}

We generate 60 sets of agents. For each set of agents we generate and run new sets of items 5 times. This results in 300 models for each learning method.

We first report the performance of the learned trees and linear models in terms of prediction accuracy. Given the same set of items as used during learning, we randomly generate 50 permutations of these items as orderings and compute the predicted valued of these orderings using the learned models. Thus, for every set of agents, there are $50 \times 5 \times 40 = 10000$ bids to predict. These predictions are compared with the evaluated values of the orderings by the simulator. We report the coefficients of determination, or $R^2$ scores, in Figure~\ref{fig:r2}. This coefficient is a standard measure for comparing regression models and is defined as:
\[
R^2 = 1 - \frac{\sum_{1 \leq d \leq k} (v^d - v(f_1^d,\ldots,f_m^d))^2}{\sum_{1 \leq d \leq k} (v^d - \texttt{mean}(v))^2},
\]
\noindent where $k$ is the number of samples, $v^d$ is the $d$th data value (simulator evaluation), $v(f_1^d,\ldots,f_m^d)$ the predicted value, and $\texttt{mean}(v)$ the mean of all data values. A large value (close to 1.0) means the regressor is an almost perfect predictor, and smaller values indicate worse performance. Figure~\ref{fig:r2} shows that all regression models lead to good prediction, with the lowest $R^2$ score over 0.85.  The learned trees with depth 8 give the best performance, followed by the trees with depth 5. Intuitively a larger tree may give a better prediction. The scores of \texttt{lasso2} and \texttt{lasso3} are very similar, and slightly better than \texttt{lasso1}.
This result makes sense since LASSO with a higher regularization parameter $\alpha$ (i.e., \texttt{lasso1}) implies the use of less features, and hence, may have less prediction power.  The tree with depth 3 shows much worse performance than the larger trees, and it is on average worse than the three linear regression models. 

\begin{figure}[thb]
\begin{center}
\includegraphics[width=.8\columnwidth]{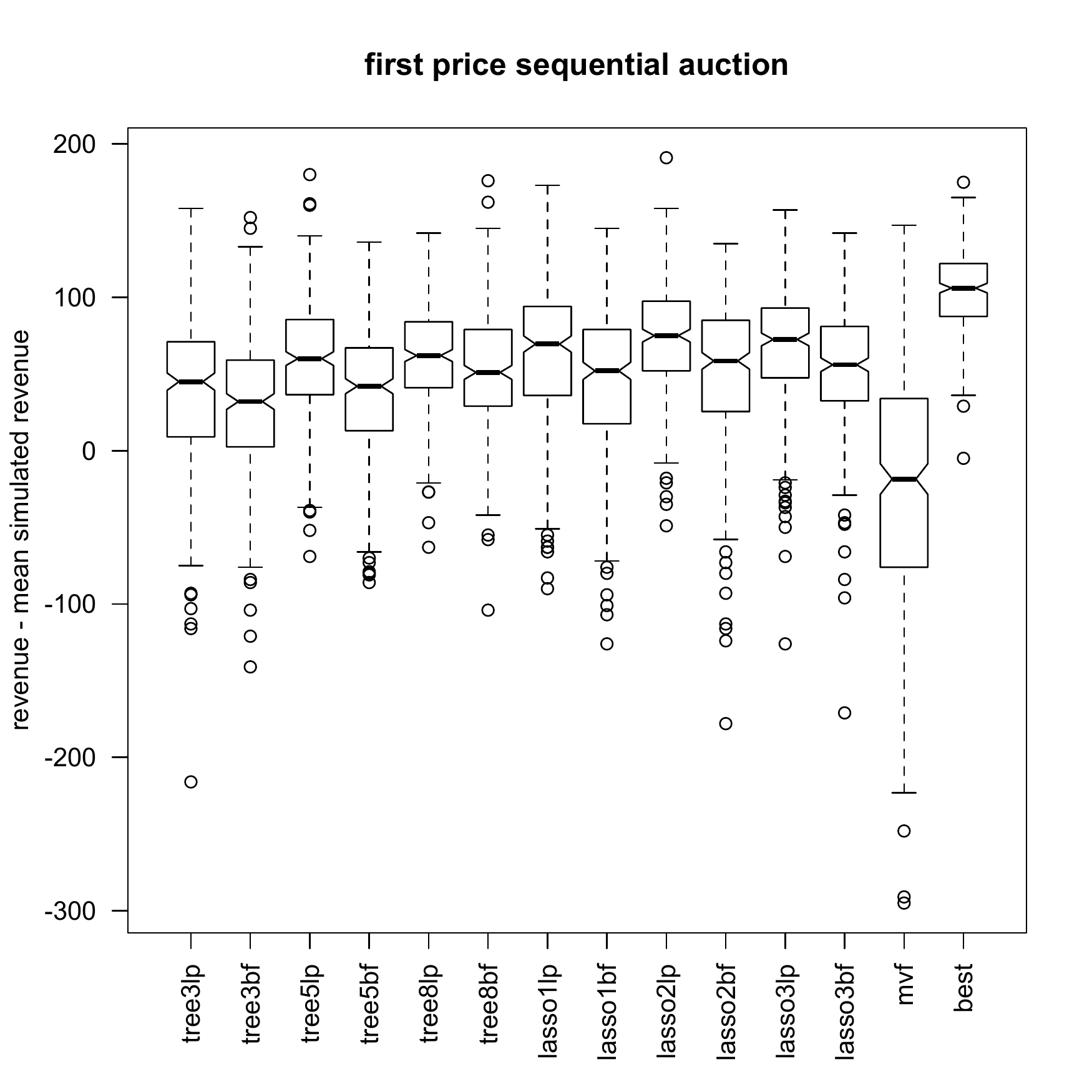}
\end{center}
\caption{The performance of the different ordering methods evaluated by the simulator, compared to \texttt{mean5000}. The simulated auctions are first-price auctions. Each box contains 300 values.   \label{fig:normalboxplot}}
\end{figure}

We now discuss the actual performance of the different ordering methods. After every ordering method returns its best ordering, these orderings are then evaluated in the simulator to get corresponding revenues. As we ran 300 different instances (60 sets of different agents, each with 5 sets of different items),  each method has 300 such revenues. We calculate the frequencies of wins by comparing the revenues in pairs in Table~\ref{tab:wins}. One obvious conclusion from the table is that the ordering heuristic \texttt{mvf} (i.e., most valuable item first) performs worst, regardless of which method it compared to. In fact, this heuristic performed even worse than the random ordering strategy \texttt{mean5000} (123 wins vs. 174).\footnote{Notice that there is one instance where \texttt{mean5000} is better than \texttt{best5000}. This is due to the random scheme that we used in selecting winners who give two identical bids. Consequently, evaluating the same ordering twice in the simulator may result in two different revenues. We want to point out that this does not happen  often and the effect is often negligible.}
This result contradicts the theoretical finding that was concluded using much simpler auction settings. Another observation is that given the learned models, the developed white-box methods win over the black-box methods more than half of the time. This holds consistently for all 12 proposed methods (wins lp vs. bf: 171, 197, 178, 185, 195, 189). It shows that our new way of utilizing the internal structure of the learned models for optimization is promising.

\begin{figure}[thb]
\begin{center}
\includegraphics[width=.8\columnwidth]{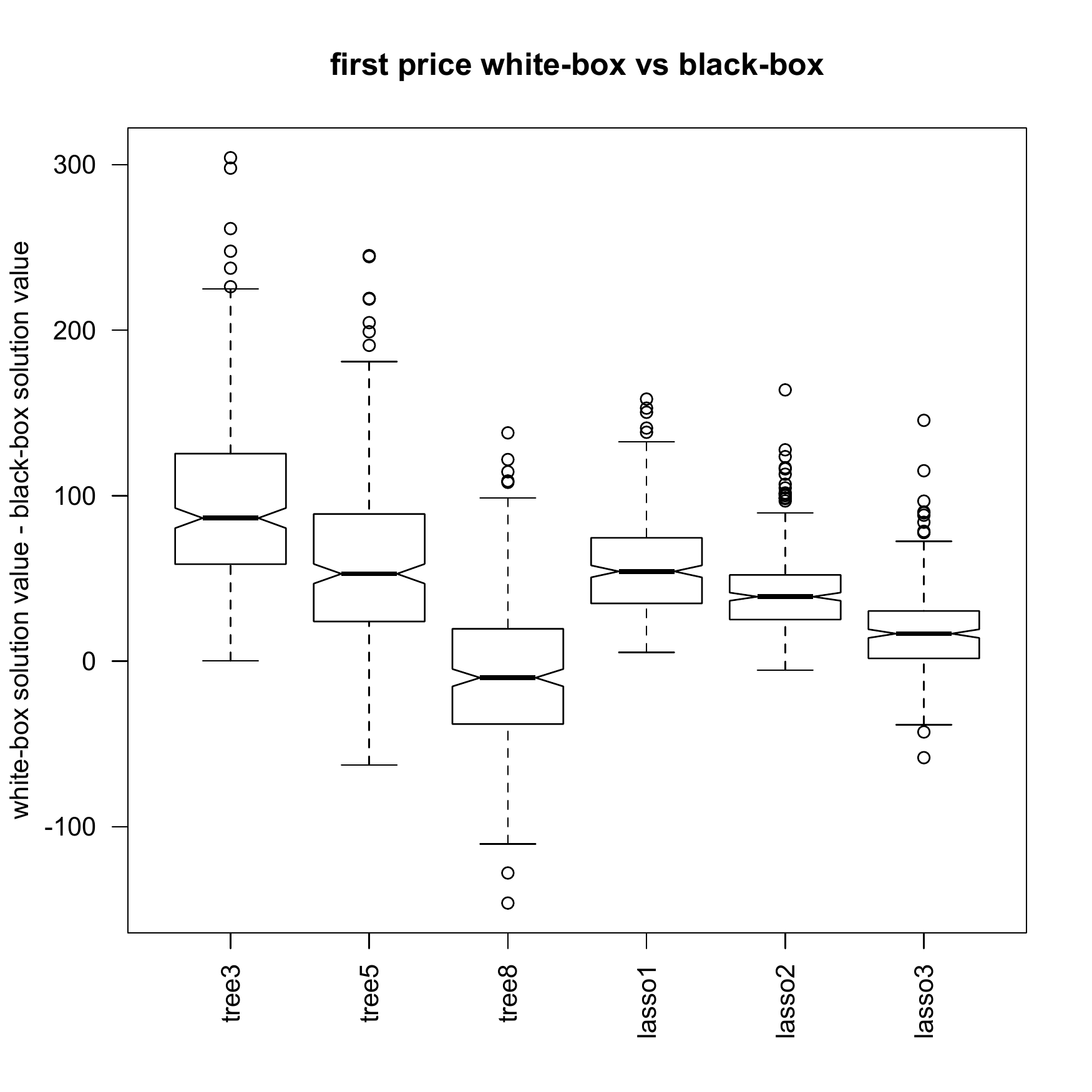}
\end{center}
\caption{The performance difference between the LP model and the best-first search, evaluated by the predictive model. Each box contains 300 values. The simulated auctions are first-price auctions.   \label{fig:normalmodel}}
\end{figure}

\begin{figure}[thb]
\begin{center}
\includegraphics[width=.8\columnwidth]{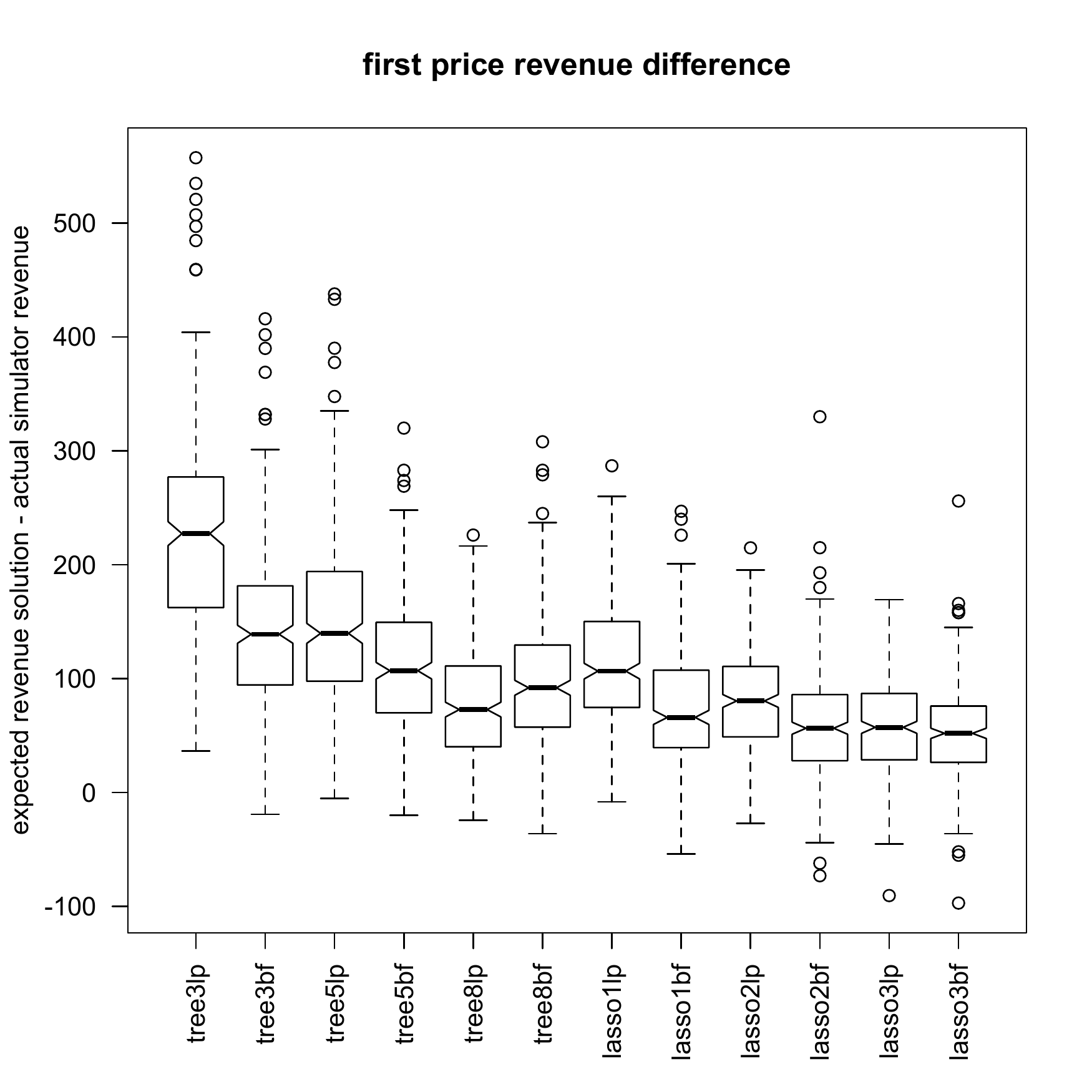}
\end{center}
\caption{The performance difference between the values evaluated by the model and the values evaluated by the simulator. Each box contains 300 values. The simulated auctions are first-price auctions.   \label{fig:normaldiff}}
\end{figure}

If we look at the results of the white-box methods built from the learned regression trees, i.e., \texttt{tree3lp}, \texttt{tree5lp}, \texttt{tree8lp}, we notice that \texttt{tree5lp} and \texttt{tree8lp} performed similarly and they are slightly better than \texttt{tree3lp}. This result is consistent with the higher $R^2$ scores of the larger trees.
Interestingly, despite their lower $R^2$ scores, the linear regression LP methods return good orderings especially \texttt{lasso2lp} and \texttt{lasso3lp} which are on average better than the three regression tree LP models. These are further confirmed with Figure~\ref{fig:normalboxplot}, which depicts the revenue differences in the simulator between the ordering methods and \texttt{mean5000}. It is more obvious from this figure that the proposed linear regression LP models are the better optimization methods than the tree LP models in practice. We believe that this is due to the cascading inaccuracies, caused by the \texttt{sum} feature that relates the predicted value to the predictions of earlier auctioned items (see discussions in Section~\ref{sec:power}). Due to the crisp boundaries in regression trees, the effect of these errors on the solution evaluation is much greater than using linear regressors. We believe the effect is smaller for larger trees because they are more accurate.

Figure~\ref{fig:normalmodel} shows the performance difference between the LP model  and the best-first search, which are built upon the same learned regression model. The solutions are evaluated using the predictive models. The value differences between the white-box and the black-box methods are more significant on smaller trees than on bigger trees (\texttt{tree3} vs \texttt{tree5} vs \texttt{tree8}), and on linear models with higher regularization parameter than on lower regularization parameter (\texttt{lasso1} vs \texttt{lasso2} vs \texttt{lasso3}). This trend shown in the results is somehow expected. The smaller tree leads to a smaller LP model, which is easier to optimize by the solver and consequently gives a much better performance than the best-first search. Similarly, the linear regression with a higher regularization parameter $\alpha$ implies less feature values to compute during white-box optimization, and therefore, its advantage over the black-box method is more obvious than \texttt{lasso2} and \texttt{lasso3} which are with smaller values of $\alpha$.

We observe that all white-box LP methods are better than the black-box methods, except the LP models resulting from the trees with depth 8. The depth 8 regression trees perform better for best-first search when evaluated on the model (Figure~\ref{fig:normalmodel}), but better for LP when evaluated in the simulator (see Figure~\ref{fig:normalboxplot}). The most likely reason for the strange behavior of the depth 8 trees is that
the best-first search outperforms LP on harder to predict instances.\footnote{Another possible cause is that the LP solver makes small rounding errors. Such errors are unavoidable because using both very small and very large coefficients and/or precision in one problem formulation can cause numerical instability. We therefore round the mean values of the leaf predictions to two digits after the decimal point in the regression tree models before translating it to LP. Unfortunately, due to the crisp decision boundaries in regression trees, these small errors can sometimes have a large effect.} Intuitively, because harder-to-predict instances typically result in larger models, they are harder to optimize in CPLEX. We checked this cause by investigating whether the $R^2$ scores of the depth 8 regression tree are correlated with which method performing better in the model evaluation. The mean of these $R^2$ scores are 0.950 when the LP performs better, and 0.942 when the best-first performs better. Although this difference seems small, it is significant.

Moreover, 
the small difference in $R^2$ scores between trees with depth 5 and depth 8 also has a significant effect on the difference between their model and simulator evaluations (e.g., due to cascading errors). 
To demonstrate this, we report in Figure~\ref{fig:normaldiff} the solution differences of the same ordering methods when being evaluated by the model and the simulator. The purpose of this comparison is to test whether the predicted outcome (using learned models) corresponds to the actual outcome (using simulator). This test is important as in general, there are no simulators available to evaluate the solutions.
The figure demonstrates that the linear regression based optimization methods return more reliable solutions, i.e., their solutions evaluated on the learned linear models are closer to the solution values returned by the simulator. Note that it is logical that most values are over estimated because we the optimization tries to solve a maximization problem.
The trees with depth 5 and 8 show a significant difference in this evaluation. The depth 3 trees end up with the highest evaluation difference, and overestimate the solution values the most.

\begin{figure}[thb]
\begin{tabular}{cc}
\includegraphics[width=0.5\textwidth]{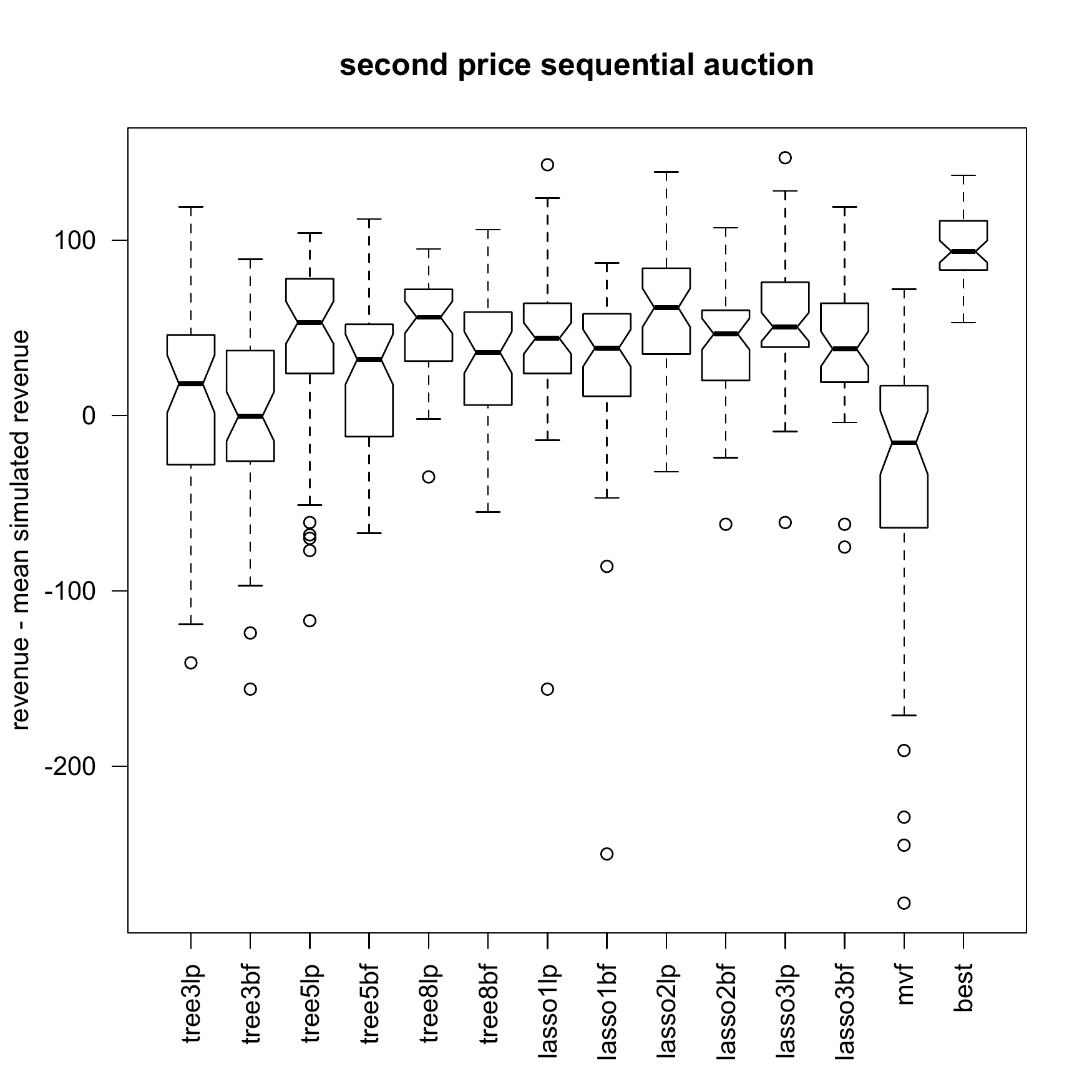}
&
\includegraphics[width=0.5\textwidth]{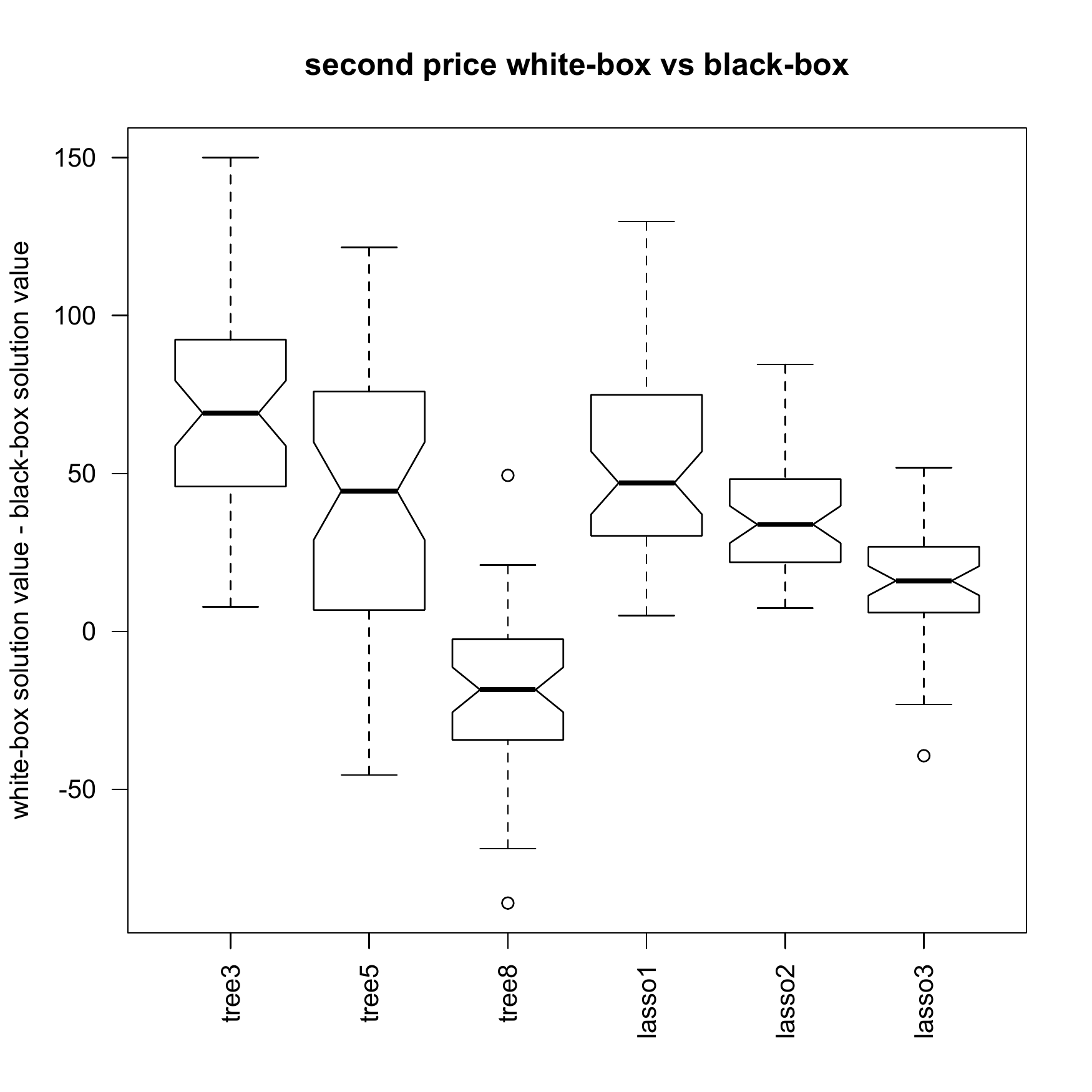}
\end{tabular}
\caption{Performance of the different ordering methods with second-price auctions and truth-telling, myopic agents. The two figures show the performance evaluated by the simulator and the model respectively. Each box contains 50 values.\label{fig:secondprice}}
\end{figure}

\begin{figure}[thb]
\begin{tabular}{cc}
 \includegraphics[width=0.5\textwidth]{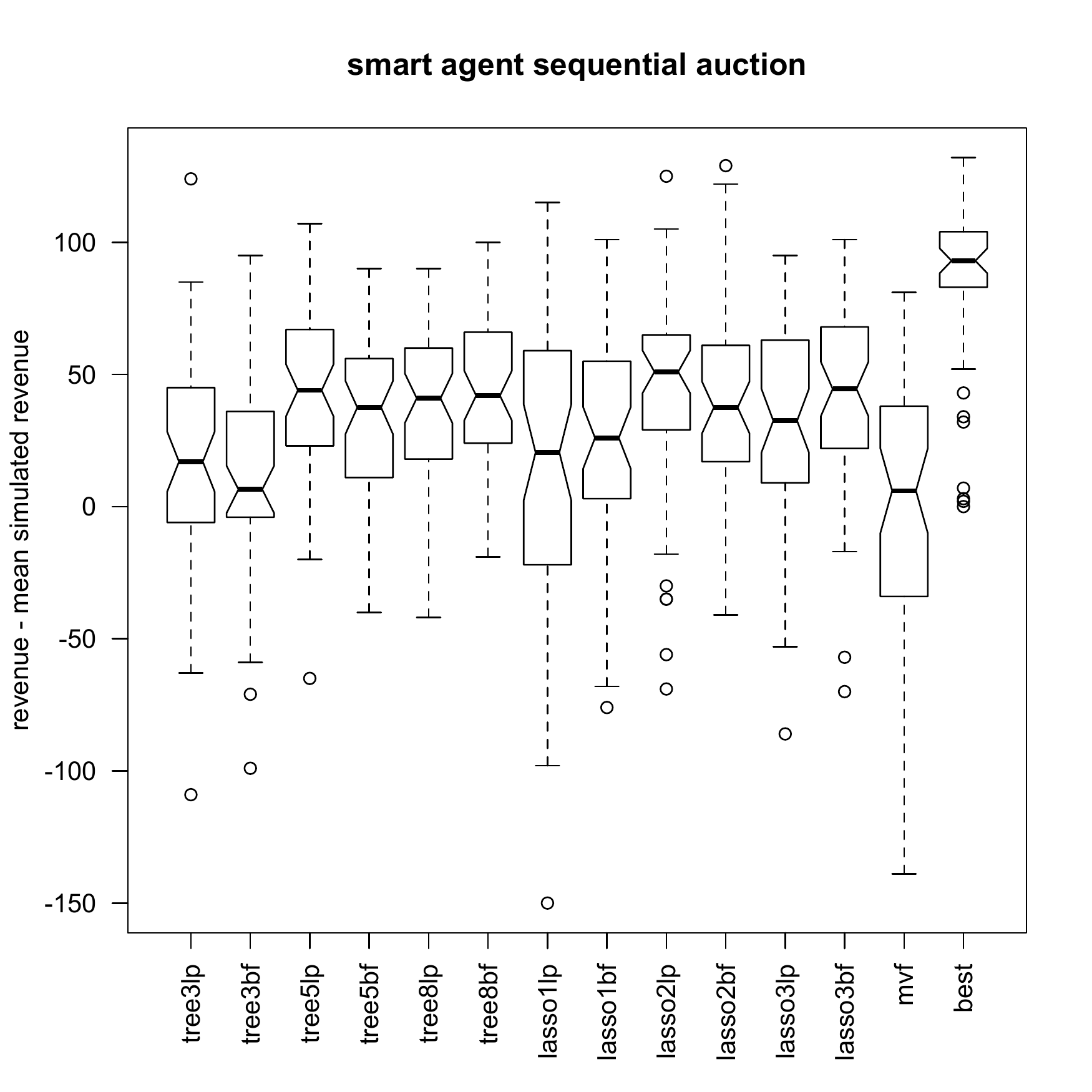}
&
\includegraphics[width=0.5\textwidth]{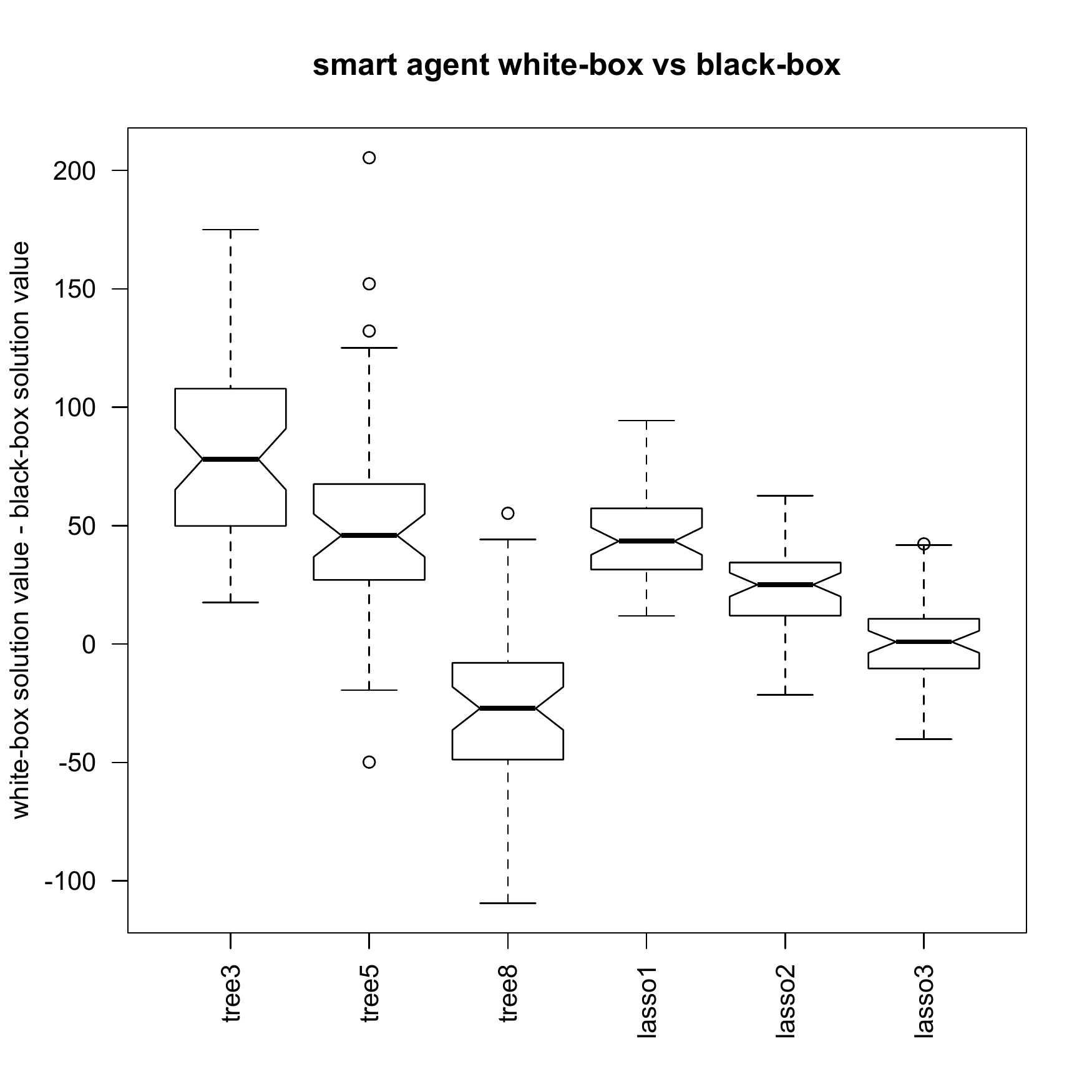}
\end{tabular}
\caption{Performance of the different ordering methods with second-price auctions and smart agents. The two figures show the performance evaluated by the simulator and the model respectively. Each box contains 50 values.\label{fig:smart}}
\end{figure}

\subsection{Experiment 2 and 3}\label{sec:exp23}

In order to demonstrate that our method of auction optimization using learned models is robust to the used auction rule or the bidding strategies, we test it in a second-price auction in Experiment 2, and with smart agents in Experiment 3. We generate 10 sets of agents. For each set of agents we run new sets of items 5 times.
Figures~\ref{fig:secondprice} and~\ref{fig:smart} show the same plots for these settings as Figures~\ref{fig:normalboxplot} and~\ref{fig:normalmodel} for the setting in the first experiment.

In the second experiment, we test with truthful-telling bidders in second-price auctions. The only difference with the setting in the first experiment is that here agents bid their true values on the items and pay an amount equal to the second highest bid, or the reserve price if there is only one bidder. As Figure~\ref{fig:secondprice} shows, the results are very similar to those in the first experiments: (1) all methods outperform the naive ordering strategies from the literature, (2) the white-box outperforms the black-box methods, and (3) the linear regression models perform best. However, the difference between the best performing methods \texttt{tree5} and \texttt{lasso2} is no longer significant.

In the third experiment, we test with smart agents that aim to maximize their final utility in a second-price auction. 
Specifically, when deciding what value to bid on the current item $r_i$, they have access to the auction simulator and use it to run the remaining items $I'\setminus r_i$. For computational reasons, when running the remaining items $I'\setminus r_i$, it is assumed that every agent bids truthfully and pays according to the second-price rule in these runs. For every bid $r_i$, they run the simulator twice: once in the situation where they bid the item $r_i$ with their reservation prices (\texttt{run1}), and once where they do not buy the item (\texttt{run2}). They then decide what value to bid according to the following rules:
\begin{itemize}
\item If after \texttt{run2} the agent has a remaining budget greater than its value for the item, it bids truthfully. The intuition is that it is better to buy an item than to have remaining budget.
\item Else, if the total utility after \texttt{run2} is less than after \texttt{run1}, it also bids truthfully. When it is better to buy an item, try to obtain it.
\item Else, it bids its true value minus the difference in utility after \texttt{run2} and \texttt{run1}. When it is $x$ monetary units better to not buy an item, try to obtain it for the value minus $x$.
\end{itemize}
Using these rules, the agents bid the highest value that they expect will give them an increase in utility. 

As can be seen in Figure~\ref{fig:smart}, also in this challenging setting, our method performs significantly better than the naive ordering rules from the literature. There is however a much larger variance in the performance of the different methods, causing the difference between black-box and white-box to be insignificant in the simulator. When evaluated on the model, however, white-box is still better than black-box in all cases except \texttt{tree8}. An interesting final observation is that, with these smart agents, the \texttt{mvf} method does seem to perform slightly better than \texttt{mean5000} (although not significantly), while in the other experiments it performed consistently worse.

\section{Comparison to related work}\label{sec:related}

We discuss related works and how our work contributes to and from several related research communities. 

\subsection{Interplay between mathematical optimization and machine learning}

Many studies have investigated the interplay of data mining and machine learning with mathematical modeling techniques, see overview in e.g.~\cite{Bennett06,Meisel10,corne:synergies}.
Most of these investigate how to use data mining to estimate the value of parameters in decision making models or to replace decision model structure when it cannot be fully determined from the hypotheses at hand. For instance, Brijs et al.~\cite{BrijsSVW04} build a decision model as an integer program that maximizes product assortment of a retail store. The decision model is then refined by incorporating additional decision attributes that are the learned patterns from recorded sales data. Li and Olafsson~\cite{LiO05} use a decision tree to learn dispatching rules that are then used to decide which job should be dispatched first. These dispatching rules are previously unknown, and it is assumed that it is worthwhile to capture the current practices from previous data.  Gabel and Riedmiller~\cite{Gabel08} model production scheduling problem as multi-agent reinforcement learning where each agent makes its dispatching decisions using a reinforcement learning algorithm based on a neural network function approximation. 

Another line of work investigates how to use learning techniques during optimization in order to learn properties of good solutions. For instance, Defourny et al.~\cite{Defourny12scenario} combine the estimation of statistical models for returning a decision rule given a state with scenario tree techniques from multi-stage stochastic programming. This line of work shares similarities with the field of black-box optimization, see, e.g.,~\cite{jones1998efficient, shan2010survey, rios2013derivative}. In black-box optimization, methods are used to approximate a function with unknown analytical form and which typically is expensive to execute. In contrast, in multi-stage stochastic programming this form is known but stochastic. An often applied technique for black-box optimization is the use of surrogate methods, see, e.g.,~\cite{koziel2011surrogate}. Surrogates are approximations of the black-box function that are less expensive to execute, typical examples include linear/polynomial regression, neural networks, and other methods from machine learning. These functions are trained during optimization from (as few as possible) black-box function calls.

As learning tasks can lead to challenging optimization problems, researchers have also applied mathematical optimization methods in order to increase learning efficiency. For instance, Bennett et al.~\cite{Bennett93bilinearseparation} use linear programming for determining linear combination splits within two-class decision trees.
Chang et al.~\cite{ChangRaRo12} propose a Constrained Conditional Model (CCM) framework to incorporate domain knowledge into a conditional model for structured learning, in the form of declarative constraints. CCMs solve prediction problems.
In~\cite{UneyT06}, the authors build a mixed integer program for multi-class data classification. A comprehensive overview of optimization techniques used in learning is given in~\cite{sra2012optimization}.
Researchers are also interested in using mathematical optimization methods in order to find entire models and rules, see e.g.,~\cite{Carrizosa13,Raedt10constraint,heuleverwer}.

Our approach fits in the first line of research of this interplay. The proposed best-first search method uses regression models to learn good orderings, which is then applied during search to evaluate the solutions of OOSA. Hence, similar to the existing work, the models learned from data are used in a black-box fashion. This approach shares similarities with surrogate methods for black-box optimization. An important difference is that the (surrogate) models here are learned from data. 

Furthermore, different from the existing work, our proposed white-box optimization method makes all the properties of the learned models visible to the optimization solver. To the best of our knowledge, this way of using the results of machine learning is entirely novel. Moreover, it realizes the construction of an optimization model automatically from data, providing a new way of modeling in mathematical optimization.

\subsection{Sequence models}\label{sec:mdp}

As an auction ordering is essentially a sequence of items, our work is also related to the many machine learning approaches for sequence modeling. However, to the best of our knowledge none of the existing sequence models fits our auction setting.
Language models such as deterministic automata~\cite{delaHiguera10grammatical} are too powerful since they can model every possible sequence independently and therefore require too much data to learn accurately. Short sequence models such as hidden Markov models or N-grams~\cite{Bishop06pattern} do not model the dependence on items sold a long time (more than the sliding window length) before.

Markov decision processes (MDPs) (see, e.g.,~\cite{puterman2009markov}) may be closest to our auction setting, as they can directly model the expected price per item and come with methods that can be used to optimize the expected total reward (revenue). However, we notice that a straightforward implementation of the auction design problem as an MDP is not possible.
Let us try to model auction design as an MDP. Because of the Markov assumption, every state in this MDP has to contain all the relevant information for the auctioneer's decision on which item to auction: the set of available items, the bidders' valuation functions, budgets and strategies, and for every bidder the items (s)he already possesses.
In every state $q$ of this process, the auctioneer can choose what item $i$ to put to auction from a multiset of available items $I$. The next state $q'$ resulting from auctioning item $i$ depends on the bidders and their valuations. These are unknown to the auctioneer, but probabilities can be used to estimate them. These probabilities $P_i(q,q')$ provide a distribution over the possible next states and the corresponding rewards $R_i(q,q')$, given $i$ is auctioned. In every possible next state $q'$, the set of available items is equal to $I - \{i\}$, i.e., equal to the items in $q$ minus the sold item $i$. The goal of the auctioneer is to maximize the expected rewards (revenue) for a given set of items $I$. In every state $q$, (s)he thus has to take an action (choosing an item) that maximizes the sum of the expected rewards $V(q,I)$ of items in $I$ starting in state $q$:

\[
V(q,I) = \arg\max_{i \in I} \left( \sum_{q'} P_i(q,q')R_i(q,q') + V(q',I-\{i\}) \right).
\]

In this equation, we separated $I$ from $q$ to highlight the major hurdle that needs to be overcome in order to represent the auction design problem as an MDP. $I$ needs to be included in the MDP since it determines the set of available actions in every state. However, since the set of items is finite this makes the MDP acyclic and at least as large as the number of possible subsets of items from $I$ (assuming the effect of their ordering is represented differently), i.e., at least $2^{|I|}$. In order to learn the rewards and transition probabilities, an auctioneer would therefore need an extremely large data sample. 


This intuitively shows why it is difficult to represent the auction design problem as an MDP. However, with a suitable factored representation of the states and/or function approximation~\cite{puterman2009markov} of the rewards, it could be possible to represent our auction problem as an MDP. In this case, a major hurdle will be to find a representation that results in Markovian states, which is needed to apply the dynamic programming methods. Since the problem of deciding whether good auction ordering exists is NP-complete (Theorem~\ref{thrm:complexity}), and these methods run in polynomial time, this is impossible without an exponentially large state space unless $\P = \NP$. Our method relies on solvers and search methods for NP-complete problems, making a polynomial state space possible, and therefore requiring much less data to estimate the model parameters.

\subsection{Auction design}
In the auction literature, a few existing papers investigate the impact of ordering on the performance of sequential auctions.  
One line of related research focuses on theoretical analysis. In the economics literature (see~\cite{Elmaghraby03importance,Pitchik09budget,Subramaniam09optimal}), such theoretical studies were typically carried out under very restricted markets. The main research focus there is to analyze equilibrium bidding strategies of bidders who compete for (usually) two items (heterogeneous or homogeneous), and then to gain insights on the impact of ordering on the auction outcome based on derived bidding behaviours.
For instance, Elmaghraby~\cite{Elmaghraby03importance} studies the influence of ordering on the efficiency of the sequential second price procurement auctions, where a buyer outsources two heterogeneous jobs to suppliers with capacity constraints. Suppliers can only win 1 job in this setting. The author shows that specific sequences lower procurement costs and identifies a class of bidders' cost functions where the efficient orderings (i.e. the auction rewards the jobs to the suppliers with the lowest total costs) and equilibrium bidding strategies exist.
Pitchik~\cite{Pitchik09budget} points out that in the presence of budget constraints, a sealed-bid sequential auction with two bidders and two goods may have multiple symmetric equilibrium bidding functions, and the ordering of sale affects the expected revenue. If the bidder who wins the first good has a higher income than the other one, the expected revenue is maximized.
Subramaniam and Venkatesh~\cite{Subramaniam09optimal} investigate the optimal auctioning strategy of a revenue-maximizing seller, who auctions two items, which could be complements or substitutes.
They show that when the items are different in value, the higher valued item (among the two) should be auctioned first in order to increase the seller's revenue. 
A similar revenue-maximizing strategy is proposed by Benoit and Krishna~\cite{Benoit01multiple} in a complete information auction setting. The authors conclude that in such a setting, when selling two items to budget constrained bidders, it is always better to sell the more valued item first. However, this strategy does not optimize the revenue anymore when more than two items are to be auctioned.

In the computer science literature,
Elkind and Fatima~\cite{Elkind07maximizing} study how to maximize revenue in sequential auctions with second-price sealed-bid rules, where bidders are homogeneous, i.e., all their valuations are drawn from public known uniform distributions, they want to win only one item (but they can bid any of items). 
In this setting, the authors analyze the equilibrium bids, and develop an algorithm that finds an optimal agenda (i.e., ordering). 
Vetsikas et al. \cite{VetsikasJ09} study a similar auction setting, but unlike~\cite{Elkind07maximizing}, they assume the valuations are known to the bidders at the beginning of the auction. The focus of their work was to compute the equilibrium strategies for bidders. Later, Vetsikas \cite{Vetsikas13} analyzes the bidding strategies for budget constrained bidders in sequential Vickrey auctions. However, it is a challenge to compute the equilibrium strategies in practice.

Several empirical research has been conducted to test the theoretical findings in the economics community. Grether et al.~\cite{Grether09} report on a field experiment that tests the ordering strategies of a seller in sequential, ascending automobile auctions. 
They conclude that the worst performing ordering in terms of revenue is for the seller to auction vehicles from highest to lowest values. 
Raviv~\cite{Raviv06} uses an example to demonstrate that there are cases where the ordering of the auction does not affect revenue using a second price seal-bid auction. In his example, there are two heterogeneous items for sale among three risk neutral bidders. Raviv shows that when the ordering is randomized, the expected selling price stays the same, no matter which one of the items is sold first.

We are not the first who consider learning from the previous auctions. However, the difference lies in the fact that the most existing work study how bidders learn from the past information, and update their bids.
Boutilier et al.~\cite{Boutilier99sequential} propose a learning model for bidders to update their bidding policies in sequential auctions for resources with complementarities. The bidding strategies are computed based on the estimated distribution over prices, that is modeled by dynamic programming. 
Goes et al.~\cite{Goes10} present an empirical study of real sequential online auctions. They analyze the data from an online auction retailer, and show that bidders learn and update their willingness to pay in repeated auctions of the same item.
In~\cite{Pinker10}, the authors show the benefits of using earlier auction data for the management of sequential, multi-unit auctions, where the seller needs to split its entire inventory into sequential auctions of smaller lots in order to increase its profit. In their work, an auction feedback mechanism is developed based on a Bayesian model, and it is used to update the auctioneer's beliefs about the bidders' valuation distribution.

Our contribution to the auction literature lies on the fact that our approach can be applied to design optimal auctions based on historical auction data. The advantage of using machine learning and data mining methods is that they are robust to the uncertainty (or noise), and hence have high potential to be applied to real-world auction design. Moreover, the approach itself is general and can be applied to many different auction optimization problems, such as finding best reserve price for items for sale, or maximizing social welfare instead of revenue. The only necessary change is on the selection of the features for learning regression models.

\section{Conclusions and future work}
Mathematical optimization relies on the availability of knowledge that can be used to construct a mathematical model for the problem at hand. This knowledge is not always available. For instance, in
multiagent problems,
agents are autonomous and
often unwilling
to share their local information. Frequently, this autonomy and private information influence the outcome of the optimization, making finding an optimal solution very difficult. 
In this paper, we adopt the idea of using machine learning techniques to estimate these influences for an optimization problem with many unknowns: the optimal ordering for sequential auctions (OOSA) problem.

We have demonstrated our approach by transforming historical auctions into data sets for learning regression trees and linear regression models, which subsequently are used to predict the expected value of orderings for new auctions.
We proposed two types of optimization methods with learned models, a black-box best-first search approach, and a novel white-box approach that maps learned models to integer linear programs (ILP).
We built an auction simulator with a set of bidder agents to simulate an auction environment. The simulator was used for generating historical auction data, and for evaluating the orderings of items returned by our methods. We ran an extensive set of experiments with different agents and bidding strategies. Although optimizing the orderings in sequential auctions is a hard problem, our proposed methods obtained very high values, significantly outperforming the naive methods proposed in the literature.
The experimental results also demonstrate the advantage of using the white-box method for optimization, which significantly outperforms the black-box approach in nearly all settings. In addition, they indicate that when the learned model becomes more complex,
it potentially results in more constraints and consequently, an increase in the time needed to solve the problem in a white-box fashion. Since more complex models are (potentially) better predictors, this shows a clear trade-off between modeling and optimization power in white-box optimization. In our opinion, the benefits of the white-box approach largely outweigh the benefits of using black-box optimization.

Besides an improved performance, a very big benefit of the white-box formulation is that it
provides a new way of obtaining traditional mathematical models.
Our method therefore has many other potential application areas, especially in problems where more and more data is being collected. 
Even in cases where there already exists a handcrafted optimization model, a model that is learned and translated using our method can easily be integrated into existing (I)LP formulations in order to determine part of the objective function based on data. In this way, one can combine the vast amount of expert knowledge available in these domains with the knowledge in the readily available data. We would like to investigate how such integration works in the future.


We chose a relatively simple auction model for ease of explanation in this paper. However, our approach works whenever regression models are able to provide reliable predictions of the bidding values. Hence we believe it can be applied to other auction formats with more complex valuation functions (i.e. combinatorial preferences~\cite{cramton2006combinatorial}) and more complex bidding strategies. In the future, we plan to test our approach on real auction data.

Our experiments also highlight some interesting properties of the white-box method. Firstly, they show an improvement in performance when the number of features is reduced and/or the models are less complex. It would therefore be very interesting to investigate the effect of pruning and feature selection or reduction on the performance of our methods. Secondly, they show a tendency of the regression tree optimizer to overestimate, i.e., find orderings that have a much higher expected revenue than its revenue in practice. Intuitively, the solver abuses the crisp nature of the regression tree in order to find a solution that satisfies exactly the right constraints. Part of the problem is that, although these constraints are learned from data, and therefore uncertain, the solver treats them as exact. Fortunately, there exists a long history of methods that try to optimize in the presence of such uncertainties in the area of robust optimization, see, e.g.,~\cite{BEN:09}. In future work, we will investigate the potential uses of these techniques for learned models. 

Recently, regression tree models with linear models in the leaf nodes have also been successfully used as black-box surrogate functions~\cite{verbeeck2013multi}. Since it is also straightforward to translate these trees given our two encodings (replace the leaf variables by indicators for which linear function to use), it would be very interesting to investigate the possibility of a white-box alternative.

\appendix

\section{Hardness of auction design using learned predictors}\label{sec:proofs}

We show that using predictive models instead of agents with utility functions does not reduce the complexity of the problem: it remains NP-complete for both regression trees and linear regression predictors.

\begin{lemma} Using regression trees, the problem of whether there exists an ordering that has a total predicted value of at least $K$ is NP-complete.
\label{lem:tree}
\end{lemma}
\begin{proof} The proof follows from the fact that we can use simple regression trees to model the preferences of the two agents from Theorem 1, and evaluating an ordering using these trees can be done in polynomial time. The regression tree for every item type $r_i$ is shown in Figure~\ref{fig:trees1}.
\begin{figure}[thb]
\centering
{\scriptsize
\begin{tikzpicture}[>=stealth]
    \node [ellipse,draw,text centered] (root) at (0,0) {$\texttt{sum} \leq  \sum I$};
    \node [ellipse,draw=none,text width=4em, text centered] at (0,0.5)  {Type $r_i$};
    \node [ellipse,draw,text centered] (left) at (-2,-1)  {$\texttt{sum} \leq \sum I - v_2(r_i)$};
    \node [ellipse,draw,text centered] (leftright) at (0,-2)  {predict $v_2(r_i)$};
    \node [ellipse,draw,text centered] (leftleft) at (-4,-2)  {predict $0$};
    \node [ellipse,draw,text centered] (right) at (2,-1)  {predict $v_1(r_i)$};
   \draw[->] (root) edge node [draw=none,left] {yes} (left);
   \draw[->] (left) edge node [draw=none,left] {yes} (leftleft);
   \draw[->] (left) edge node [draw=none,right] {no} (leftright);
   \draw[->] (root) edge node [draw=none,right] {no} (right);
\end{tikzpicture}
}
\caption{\label{fig:trees1} The regression tree for every item type $r_i$ used to model a partitioning problem.}
\end{figure}
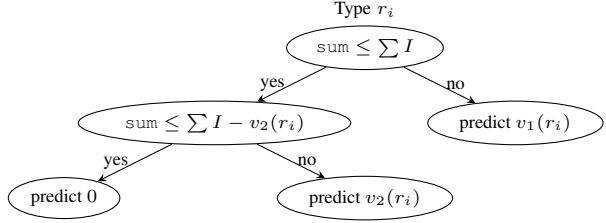
\end{proof}

\begin{lemma} Using linear regression predictors, the problem of whether there exists an ordering that has a total predicted value of at least $K$ is NP-complete.
\label{lem:linear}
\end{lemma}
\begin{proof} We prove the lemma using a construction for computing the value of a quadratic function using only linear functions, the ordering problem, and our feature values. The maximum value of this quadratic function is then forced to coincide with the solution of a partition problem instance: Given a set of integers $I = \{ v_1, \ldots, v_n \}$, is $I$ dividable into two sets $A$ and $B$ such that $\sum A = \sum B$?

From the partition instance, let $k = \frac{1}{2} \sum_{1 \leq i \leq n} v_i$, we construct the following items and linear regression predictors (functions $v()$):
\begin{itemize}
\item $n$ items of type $x_1 \ldots x_n$, with $v(x_i) = v_i - \frac{v_i \cdot \texttt{sum}(y)}{2k}$, and
\item $1$ item of type $y$, with $v(y) = 2k - \sum_{1 \leq i \leq n} \texttt{sum}(x_i)$.
\end{itemize}
The objective is to maximize $\sum_{1 \leq i \leq n} v(x_i) + v(y)$. The corresponding decision problem is to ask whether there exists an ordering that achieves a value of $2\frac{1}{2}k$.

\paragraph($\Rightarrow$) Let $A, B$ be a partition of $I$ such that $\sum A = \sum B$, and let $a_1, \ldots,a_{|A|}$ and $b_1, \ldots, b_{|B|}$ be the corresponding items of type $x_1 \ldots x_n$. The following ordering then gives a value of $\sum_{1 \leq i \leq n} v_i$:
\[
a_1~\ldots~a_{|A|}~y~b_1~\ldots~b_{|B|}
\]
In this ordering, $\sum_{1 \leq i \leq |A|} v(a_i) = k$ by definition of $A$ and since $\texttt{sum}(y) = 0$ before item $y$ is auctioned. Consequently, $\sum_{1 \leq i \leq n} \texttt{sum}(x_i) = \sum_{1 \leq i \leq |A|} v(a_i) = k$, giving $v(y) = 2k  - k = k$. $\sum_{1 \leq i \leq |A|} v(a_1) + v(y)$ thus already obtains the objective value of $2k$, and therefore $\sum_{1 \leq i \leq |B|} v(b_1)$ should be equal to $\frac{1}{2}k$. By definition of $B$, $\sum_{1 \leq i \leq |B|} v(b_i) = k - \frac{k \cdot \texttt{sum}(y)}{2k}$. Since $v(y) = k$, $\texttt{sum}(y) = k$, and thus $\sum_{1 \leq i \leq |B|} v(b_i) = k - \frac{k \cdot k}{2k} = k - \frac{k}{2} = \frac{1}{2}k$, proving that the ordering obtains a value of $2\frac{1}{2}k$.

\paragraph($\Leftarrow$) To prove the other direction, let us further analyse the relation between the objective function $\sum_{1 \leq i \leq n} v(x_i) + v(y)$ and the auction ordering. The only term in the $v(x_i)$ predictors that depends on the ordering is $\texttt{sum}(y)$, all other terms are constants. This term is equal to zero for the $x_i$ items auctioned before $y$, and equal to $v(y)$ for the items auctioned after $y$. Similar to the ($\Rightarrow$) part, let $a_1, \ldots,a_{|A|}$ denote the $x_i$ items before $y$, $b_1, \ldots, b_{|B|}$ those after $y$, and $A,B$ the corresponding partition of items in $I$. The objective value is then given by $\sum_{1 \leq i \leq |A|} v(a_i) + v(y) + \sum_{1 \leq i \leq |B|} v(b_i)$. We analyse these three parts in turn.

\begin{itemize}
\item $\sum_{1 \leq i \leq |A|} v(a_i) = \sum_{v_i \in A} v_i$ since $\texttt{sum}(y) = 0$ for these items.
\item $v(y) = 2k - \sum_{1 \leq i \leq |A|} v(a_i) = 2k - \sum_{v_i \in A} v_i = \sum_{v_i \in B} v_i$.
\item $\sum_{1 \leq i \leq |B|} v(b_i) = \sum_{v_i \in B} v_i - \sum_{v_i \in B} \frac{v_i \cdot v(y)}{2k}$, since $v(y) = \sum_{v_i \in B} v_i$, this becomes $\sum_{v_i \in B} v_i - \frac{v(y)^2}{2k}$.
\end{itemize}

Since $\sum_{v_i \in A} v_i + \sum_{v_i \in B} v_i = 2k$, the overall objective function is given by:
\[
2k + v(y) - \frac{v(y)^2}{2k}
\]
which is maximized when $v(y) = k$ (for $k > 0$) with value $2k + k - \frac{k^2}{2k} = 2\frac{1}{2}k$. This exact value of $v(y) = k$ is obtained when $\sum_{v_i \in B} v_i = k$. The sets $A$ and $B$ thus give a partition of $I$.
\end{proof}

\emph{Remarks}. We proved the NP-completeness for the general case of Lemma~\ref{lem:linear}. However, we do not know whether the complexity holds for more realistic valuation functions that bidders have.


\begin{thebibliography}{10}
\expandafter\ifx\csname url\endcsname\relax
  \def\url#1{\texttt{#1}}\fi
\expandafter\ifx\csname urlprefix\endcsname\relax\def\urlprefix{URL }\fi
\expandafter\ifx\csname href\endcsname\relax
  \def\href#1#2{#2} \def\path#1{#1}\fi

\bibitem{Gabel08}
T.~Gabel, M.~Riedmiller, Adaptive reactive job-shop scheduling with learning
  agents, International Journal of Information Technology and Intelligent
  Computing 2~(4) (2008) 1--30.

\bibitem{huyet06}
A.~Huyet, Optimization and analysis aid via data-mining for simulated
  production systems, European Journal of Operational Research 173~(3) (2006)
  827--838.

\bibitem{LiO05}
X.~Li, S.~{\'O}lafsson, Discovering dispatching rules using data mining,
  Journal of Scheduling 8~(6) (2005) 515--527.

\bibitem{Bernhardt}
D.~Bernhardt, D.~Scoones, A note on sequential auctions, American Economic
  Review~(3)  653--657.

\bibitem{Heck97experiences}
E.~van Heck, P.~M.~A. Ribbers, Experiences with electronic auctions in the
  dutch flower industry, Electronic Markets 7~(4) (1997) 29--34.

\bibitem{Pinker10}
E.~J. Pinker, A.~Seidmann, Y.~Vakrat, Using bid data for the management of
  sequential, multi-unit, online auctions with uniformly distributed bidder
  valuations, European Journal of Operational Research 202~(2) (2010) 574--583.

\bibitem{Gallien05smart}
J.~Gallien, L.~M. Wein, A smart market for industrial procurement with capacity
  constraints, Management Science 51 (2005) 76--91.

\bibitem{Elmaghraby03importance}
W.~Elmaghraby, The importance of ordering in sequential auctions, Management
  Science 49 (2003) 673--682.

\bibitem{Grether09}
D.~M. Grether, C.~R. Plott, Sequencing strategies in large, competitive,
  ascending price automobile auctions: An experimental examination, Journal of
  Economic Behavior \& Organization 71~(2) (2009) 75--88.

\bibitem{Raviv06}
Y.~Raviv, New evidence on price anomalies in sequential auctions: Used cars in
  new jersey, Journal of Business \& Economic Statistics 24 (2006) 301--312.

\bibitem{Subramaniam09optimal}
R.~Subramaniam, R.~Venkatesh, Optimal bundling strategies in multiobject
  auctions of complements or substitutes, Marketing Science 28 (2009) 264--273.
\newblock \href {http://dx.doi.org/10.1287/mksc.1080.0394}
  {\path{doi:10.1287/mksc.1080.0394}}.

\bibitem{Pitchik09budget}
C.~Pitchik, Budget-constrained sequential auctions with incomplete information,
  Games and Economic Behavior 66~(2) (2009) 928--949.

\bibitem{Elkind07maximizing}
E.~Elkind, S.~Fatima, Maximizing revenue in sequential auctions, in:
  Proceedings of the 3rd international conference on Internet and network
  economics, WINE'07, Springer-Verlag, Berlin, Heidelberg, 2007, pp. 491--502.

\bibitem{jones1998efficient}
D.~R. Jones, M.~Schonlau, W.~J. Welch, Efficient global optimization of
  expensive black-box functions, Journal of Global optimization 13~(4) (1998)
  455--492.

\bibitem{shan2010survey}
S.~Shan, G.~G. Wang, Survey of modeling and optimization strategies to solve
  high-dimensional design problems with computationally-expensive black-box
  functions, Structural and Multidisciplinary Optimization 41~(2) (2010)
  219--241.

\bibitem{Garey79}
M.~R. Garey, D.~S. Johnson, Computers and intractability -- a guide to the
  theory of NP-completeness, W.H. Freeman and company, 1979.

\bibitem{mdp}
M.~L. Puterman, Markov Decision Processes: Discrete Stochastic Dynamic
  Programming, John Wiley \& Sons, Inc., New York, NY, USA, 1994.

\bibitem{cart84}
L.~Breiman, J.~Friedman, R.~Olshen, C.~Stone, {Classification and Regression
  Trees}, Wadsworth and Brooks, Monterey, CA, 1984.

\bibitem{Tibshirani94regressionshrinkage}
R.~Tibshirani, Regression shrinkage and selection via the lasso, Journal of the
  Royal Statistical Society, Series B 58 (1994) 267--288.

\bibitem{scikit}
scikit-learn: machine learning~in {P}ython, http://scikit-learn.org/.

\bibitem{verwer2012revenue}
S.~Verwer, Y.~Zhang, Revenue prediction in budget-constrained sequential
  auctions with complementarities, AAMAS '12, 2012, pp. 1399--1400.

\bibitem{vickrey61}
W.~Vickrey, Counterspeculation, auctions, and competitive sealed tenders, The
  Journal of Finance 16~(1) (1961) 8--37.

\bibitem{Vetsikas13}
I.~A. Vetsikas, Sequential auctions with budget-constrained bidders, in: IEEE
  10th International Conference on e-Business Engineering, IEEE, 2013, pp.
  17--24.

\bibitem{cplex}
{IBM ILOG CPLEX Optimizer},
  http://www-01.ibm.com/software/integration/optimization/cplex-optimizer/.

\bibitem{Bennett06}
K.~P. Bennett, E.~Parrado-Hern\'{a}ndez, The interplay of optimization and
  machine learning research, Journal of Machine Learning Research 7 (2006)
  1265--1281.

\bibitem{Meisel10}
S.~Meisel, D.~Mattfeld, Synergies of operations research and data mining,
  European Journal of Operational Research 206~(1) (2010) 1--10.

\bibitem{corne:synergies}
D.~Corne, C.~Dhaenens, L.~Jourdan, Synergies between operations research and
  data mining: The emerging use of multi-objective approaches, European Journal
  of Operational Research 221~(3) (2012) 469 -- 479.

\bibitem{BrijsSVW04}
T.~Brijs, G.~Swinnen, K.~Vanhoof, G.~Wets, Building an association rules
  framework to improve product assortment decisions, Data Mining and Knowledge
  Discovery 8~(1) (2004) 7--23.

\bibitem{ChangRaRo12}
M.~Chang, L.~Ratinov, D.~Roth, Structured learning with constrained conditional
  models, Machine Learning 88~(3) (2012) 399--431.

\bibitem{Defourny12scenario}
B.~Defourny, D.~Ernst, L.~Wehenkel, Scenario trees and policy selection for
  multistage stochastic programming using machine learning, Journal on
  ComputingPublished online before print.

\bibitem{rios2013derivative}
L.~M. Rios, N.~V. Sahinidis, Derivative-free optimization: A review of
  algorithms and comparison of software implementations, Journal of Global
  Optimization 56~(3) (2013) 1247--1293.

\bibitem{koziel2011surrogate}
S.~Koziel, D.~E. Ciaurri, L.~Leifsson, Surrogate-based methods, in:
  Computational Optimization, Methods and Algorithms, Springer, 2011, pp.
  33--59.

\bibitem{Bennett93bilinearseparation}
K.~P. Bennett, O.~L. Mangasarian, Bilinear separation of two sets in n-space,
  COMPUTATIONAL OPTIMIZATION AND APPLICATIONS 2 (1993) 207--227.

\bibitem{UneyT06}
F.~Uney, M.~Turkay, A mixed-integer programming approach to multi-class data
  classification problem, European Journal of Operational Research 173~(3)
  (2006) 910--920.

\bibitem{sra2012optimization}
S.~Sra, S.~Nowozin, S.~J. Wright, Optimization for machine learning, Mit Press,
  2012.

\bibitem{Carrizosa13}
E.~Carrizosa, D.~Romero~Morales, Review: Supervised classification and
  mathematical optimization, Computers and Operations Research 40~(1) (2013)
  150--165.

\bibitem{Raedt10constraint}
L.~D. Raedt, T.~Guns, S.~Nijssen, Constraint programming for data mining and
  machine learning, in: AAAI, 2010, pp. 1671--1675.

\bibitem{heuleverwer}
M.~J. Heule, S.~Verwer, Exact {DFA} identification using {SAT} solvers, in:
  Grammatical Inference: Theoretical Results and Applications, Vol. 6339 of
  Lecture Notes in Computer Science, Springer Berlin Heidelberg, 2010, pp.
  66--79.

\bibitem{delaHiguera10grammatical}
C.~de~la Higuera, Grammatical Inference: Learning Automata and Grammars,
  Cambridge University Press, New York, NY, USA, 2010.

\bibitem{Bishop06pattern}
C.~M. Bishop, Pattern Recognition and Machine Learning (Information Science and
  Statistics), Springer-Verlag New York, Inc., Secaucus, NJ, USA, 2006.

\bibitem{puterman2009markov}
M.~L. Puterman, Markov decision processes: discrete stochastic dynamic
  programming, Vol. 414, John Wiley \& Sons, 2009.

\bibitem{Benoit01multiple}
J.-P. Benoit, V.~Krishna, Multiple-object auctions with budget constrained
  bidders, Review of Economic Studies 68~(1) (2001) 155--79.

\bibitem{VetsikasJ09}
I.~A. Vetsikas, N.~R. Jennings, Sequential auctions with partially
  substitutable goods., Vol.~59 of Lecture Notes in Business Information
  Processing, 2009, pp. 242--258.

\bibitem{Boutilier99sequential}
C.~Boutilier, M.~Goldszmidt, B.~Sabata, Sequential auctions for the allocation
  of resources with complementarities, in: Proceedings of the 16th
  international joint conference on Artifical intelligence - Volume 1, Morgan
  Kaufmann Publishers Inc., San Francisco, CA, USA, 1999, pp. 527--534.

\bibitem{Goes10}
P.~B. Goes, G.~G. Karuga, A.~K. Tripathi, Understanding willingness-to-pay
  formation of repeat bidders in sequential online auctions, Information
  Systems Research 21 (2010) 907--924.

\bibitem{cramton2006combinatorial}
P.~Cramton, Y.~Shoham, R.~Steinberg, Combinatorial auctions, MIT Press, 2006.

\bibitem{BEN:09}
A.~Ben-Tal, L.~El~Ghaoui, A.~Nemirovski, Robust Optimization, Princeton Series
  in Applied Mathematics, Princeton University Press, 2009.

\bibitem{verbeeck2013multi}
D.~Verbeeck, F.~Maes, K.~De~Grave, H.~Blockeel, Multi-objective optimization
  with surrogate trees, in: Proceeding of the fifteenth annual conference on
  Genetic and evolutionary computation conference, ACM, 2013, pp. 679--686.

\end{thebibliography}
\end{document}